\documentclass{article}

\usepackage{microtype}
\usepackage{graphicx}
\usepackage{subcaption}
\usepackage{booktabs} 

\usepackage{hyperref}

\usepackage[utf8]{inputenc} 
\usepackage[T1]{fontenc}    
\usepackage{hyperref}       
\usepackage{url}            
\usepackage{booktabs}       
\usepackage{amsfonts}       
\usepackage{nicefrac}       
\usepackage{microtype}      
\usepackage{xcolor}         
\usepackage{graphicx}
\usepackage{amsmath}
\usepackage{float} 
\usepackage{amssymb}

\usepackage{adjustbox}
\usepackage{graphicx} 
\usepackage{pifont}  
\usepackage{enumitem}
\usepackage{tabularx}

\usepackage{longtable}  
\usepackage{booktabs}   
\usepackage{amsmath}    
\usepackage{amssymb}    

\usepackage{booktabs} 
\usepackage{multirow} 

\usepackage{graphicx} 
\usepackage{amssymb}
\usepackage{amsthm}
\setlist{nosep}

\usepackage{xcolor}
\usepackage[dvipsnames]{xcolor}

\usepackage[preprint]{icml2026}

\usepackage{amsmath}
\usepackage{amssymb}
\usepackage{mathtools}
\usepackage{amsthm}

\usepackage[capitalize,noabbrev]{cleveref}

\theoremstyle{plain}
\newtheorem{theorem}{Theorem}[section]
\newtheorem{proposition}[theorem]{Proposition}
\newtheorem{lemma}[theorem]{Lemma}
\newtheorem{corollary}[theorem]{Corollary}
\theoremstyle{definition}
\newtheorem{definition}[theorem]{Definition}
\newtheorem{assumption}[theorem]{Assumption}
\theoremstyle{remark}

\usepackage[textsize=tiny]{todonotes}

\icmltitlerunning{Understanding Generalization from Embedding Dimension and Distributional Convergence}

\begin{document}

\twocolumn[
  \icmltitle{Understanding Generalization from Embedding Dimension and Distributional Convergence}



  \icmlsetsymbol{equal}{*}

  \begin{icmlauthorlist}
    \icmlauthor{Junjie Yu}{equal,sustech}
    \icmlauthor{Zhuoli Ouyang}{equal,sustech_2}
    \icmlauthor{Haotian Deng}{equal,sustech}
    \icmlauthor{Chen Wei}{sustech}
    \icmlauthor{Wenxiao Ma}{sustech}
    \icmlauthor{Jianyu Zhang}{sustech}
    \icmlauthor{Zihan Deng}{sustech}
    \icmlauthor{Quanying Liu}{sustech}
  \end{icmlauthorlist}

  \icmlaffiliation{sustech}{Department of Biomedical Engineering, Southern University of Science and Technology}
  \icmlaffiliation{sustech_2}{Department of Electronic and Electrical Engineering, Southern University of Science and Technology}

  \icmlcorrespondingauthor{Quanying Liu}{liuqy@sustech.edu.cn}

  \icmlkeywords{Machine Learning, ICML}

  \vskip 0.3in
]



\printAffiliationsAndNotice{\icmlEqualContribution}

\begin{abstract}
Deep neural networks often generalize well despite heavy over-parameterization, challenging classical parameter-based analyses.
We study generalization from a representation-centric perspective and analyze how the geometry of learned embeddings controls predictive performance for a fixed trained model.
We show that population risk can be bounded by two factors: 
(i) the intrinsic dimension of the embedding distribution, which determines the convergence rate of empirical embedding distribution to the population distribution in Wasserstein distance, and 
(ii) the sensitivity of the downstream mapping from embeddings to predictions, characterized by Lipschitz constants.
Together, these yield an embedding-dependent error bound that does not rely on parameter counts or hypothesis class complexity.
At the final embedding layer, architectural sensitivity vanishes and the bound is dominated by embedding dimension, explaining its strong empirical correlation with generalization performance.
Experiments across architectures and datasets validate the theory and demonstrate the utility of embedding-based diagnostics.

\end{abstract}


\section{Introduction}

Deep networks can generalize effectively even in strongly overparameterized regimes, a phenomenon that remains difficult to explain using classical capacity-based theories.
Traditional approaches based on VC dimension \citep{vapnik1994measuring, sontag1998vc} or Rademacher complexity \citep{truong2022rademacher} provide important theoretical insights, but often become vacuous at modern scales.
A key limitation of these frameworks is their primary focus on the parameter space: as model size grows, capacity measures typically scale with the number of parameters, rendering generalization guarantees increasingly loose and obscuring the mechanisms that enable large models to generalize in practice.

These challenges have motivated a growing interest in the geometry of learned hidden representations. Unlike raw parameters, deep representations are the outcome of the combined effects of data, architecture and optimization, and thus provide a functional view of the model after training. Crucially, representational analyses are less sensitive to model scale: embeddings from models of different sizes or architectures can be projected into a common ambient space (e.g., via PCA), enabling controlled comparisons that isolate geometric structure rather than parameter count. This makes representation geometry a particularly attractive lens for studying generalization across architectures and training regimes. Prior work has linked geometric properties such as clustering, separability, and consistency to generalization performance \citep{davies2009cluster, belcher2020generalisation, dyballa2024separability}. However, many existing metrics rely on label information or task-specific supervision, limiting their applicability in unsupervised or self-supervised representation learning settings.

A particularly promising direction is the study of \emph{intrinsic dimension}, a label-free measure of the geometric complexity of learned embeddings. Although modern models produce high-dimensional embeddings, empirical studies consistently find that these representations concentrate on low-dimensional structure \citep{ansuini2019intrinsic, pope2021intrinsic}. Moreover, lower intrinsic dimension has been repeatedly associated with improved generalization across architectures and training paradigms. Despite these observations, a principled explanation of why and how representation dimensionality governs generalization remains lacking. This gap motivates our work to develop a theoretical account of how the dimensional structure of learned representations influences generalization.

\begin{figure*}[!t]
  \centering
  \includegraphics[width=0.85\linewidth]{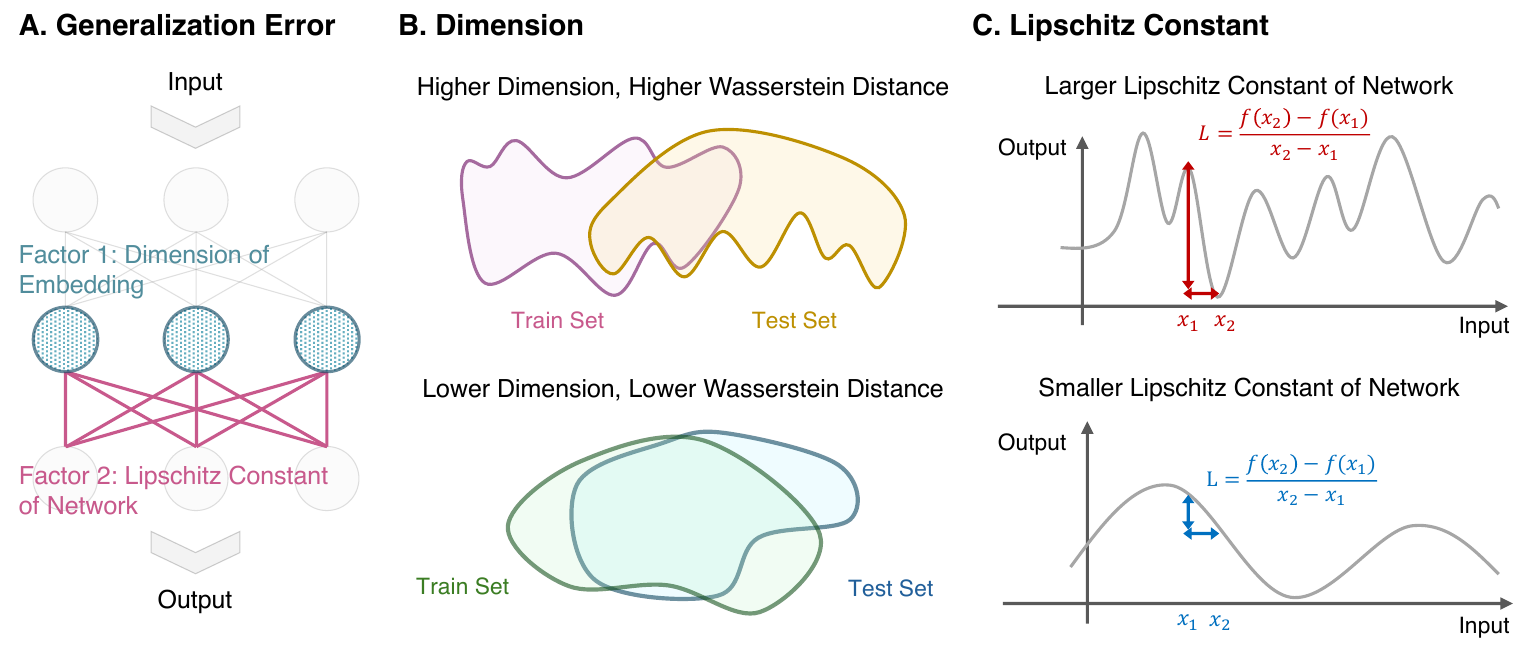}
  \caption{\textbf{Embedding Dimension and Lipschitz Constant of Network Jointly Influence Generalization Error.}
  \textbf{(A)}Generalization error depends jointly on embedding dimension and network's Lipschitz constant.
  \textbf{(B)} Lower intrinsic dimension accelerates convergence of empirical to population distribution.
  \textbf{(C)} Smaller Lipschitz constants reduce output sensitivity to perturbations.
  }
  \label{fig:motivation}
\end{figure*}

We address this question by deriving a \emph{post-training generalization bound} that makes the role of embedding dimension explicit (Figure~\ref{fig:motivation}). 
Importantly, our analysis conditions on a \emph{fixed, trained model} $F$ and does not characterize the behavior of a learning algorithm or hypothesis class. Instead, it quantifies how well the empirical distribution of learned representations approximates the true population distribution, and how this approximation error propagates to the loss.

Building on sharp Wasserstein convergence results \citep{weed2019sharp}, we show that for a fixed trained model $F$, each layer $k$ with intrinsic dimension $d_k$, sensitivity $\bar L_k$, and data-dependent constants $C_k, D_k$, the population risk $R(F)$ and empirical risk $\hat R_n(F)$ satisfy, with high probability,

\begin{align*}
R(F) \;\lesssim\; & \hat R_n(F) 
+ \bar L_k \Big(
C_k\, n^{-1/(d_k+\epsilon)} 
+ D_k \sqrt{\tfrac{1}{n}\log\tfrac{L}{\delta}}
\Big).
\end{align*}

Here we show the main form of the bound, lower-order terms are omitted for clarity.
The term $n^{-1/(d_k+\epsilon)}$ characterizes the rate at which the empirical embedding distribution converges to its population counterpart in Wasserstein distance. Consequently, for a fixed sample size, representations of lower intrinsic dimension yield more accurate distributional estimates, leading to a smaller generalization gap (Figure~\ref{fig:motivation}B). The factor $\bar L_k$ quantifies how discrepancies in the embedding distribution are amplified by downstream mappings and propagated to the loss (Figure~\ref{fig:motivation}C).

At the final layer, where the representation coincides with the model output, there is no downstream network mapping applied to the embeddings. As a result, the generalization error bound no longer depends on the architecture of subsequent layers. In this regime, the dominant term of the bound is governed by the intrinsic dimension of the output representation. Consequently, for the final layer, generalization is driven primarily by the geometry of the output space, providing a theoretical explanation for why final-layer dimension is often a strong empirical predictor of generalization
performance.

Our contributions are threefold:
\begin{enumerate}[leftmargin=*]
  \item We derive a high-probability \emph{post-training generalization bound} for a fixed model that makes the $n^{-1/(d+\epsilon)}$ dependence on intrinsic dimension explicit.
  \item We show that at the final layer the bound simplifies substantially, providing a theoretical explanation for why final-layer dimension is a strong empirical predictor of generalization.
  \item Through extensive analyses of large-scale vision and language models, we empirically verify that the predicted scaling behavior persists in realistic regimes with large models and datasets.
\end{enumerate}

\section{Related Works}

\paragraph{Classical Generalization Bounds.}
Theoretical analyses of generalization in deep learning have traditionally centered on parameter-space complexity, including VC-dimension and Rademacher complexity bounds \citep{sain1996nature, bartlett2002rademacher}, which provide worst-case guarantees based on the number of parameters. While refinements such as margin-based and norm-based bounds \citep{bartlett2017spectrally, neyshabur2015norm, neyshabur2017pac} yield tighter estimates by incorporating weight norms or spectral properties, they become vacuous in the context of modern overparameterized networks. PAC-Bayesian approaches \citep{arora2018stronger, hellstrom2025generalization, lotfi2022pac}
derive non-vacuous generalization bounds by controlling the divergence between
a data-independent prior and a data-dependent posterior distribution over
parameters. 
While these bounds can be significantly tighter in overparameterized regimes, their practical effectiveness depends critically on the choice of prior, which remains a central challenge in applying PAC-Bayesian theory to deep networks.
Overall, these theories are primarily parameter-centric and offer limited insight into how the geometric of learned representations influences generalization, despite strong empirical evidence that representation properties such as intrinsic dimension are closely tied to generalization performance.

\paragraph{Representation-Based Approaches.}  
Recent research has increasingly focused on the impact of embedding geometry on generalization. Key approaches analyze properties such as Consistency and Separability of Representations \citep{davies2009cluster, dyballa2024separability, belcher2020generalisation}. These geometric metrics offer improved interpretability and are less reliant on model scale. However, they require labeled data, limiting their applicability in scenarios like pretraining or self-supervised learning, where label information is unavailable.

\paragraph{Intrinsic Dimension of Representations.}  
A recent and promising direction in representation-based analysis is the study of \emph{intrinsic dimension}, which quantifies the complexity of embeddings. This approach aligns with the growing understanding that models implicitly compress input data during learning: a lower intrinsic dimension signifies stronger compression and has been empirically linked to improved generalization \citep{ansuini2019intrinsic, pope2021intrinsic}. While this offers a quantitative measure of how representations condense information, the theoretical mechanisms connecting intrinsic dimension to generalization remain largely unexplored.

\section{Preliminaries}
\label{sec:preliminaries}

This section introduces the key concepts, definitions and assumptions that connect representation geometry to generalization. All technical variants and detailed proofs are deferred to the appendix.

\subsection{Measures and Wasserstein distance}

\begin{definition}[Empirical measure]
Let $\mu$ be a probability distribution on a metric space $(X,d)$. 
Given $n$ i.i.d.\ samples $\{x_i\}_{i=1}^n \sim \mu$, the empirical distribution is
\[
\hat\mu_n = \frac{1}{n}\sum_{i=1}^n \delta_{x_i}.
\]
\end{definition}

\begin{definition}[Wasserstein distance]
Let $(X,d)$ be a Polish metric space, and let $\alpha,\beta$ be probability
measures on $X$ with finite first moments.
The Wasserstein distance between $\alpha$ and $\beta$ is defined as
\[
\mathcal{W}_1(\alpha,\beta)
=
\inf_{\gamma \in \Gamma(\alpha,\beta)}
\int_{X\times X} d(x,y)\, d\gamma(x,y),
\]
where $\Gamma(\alpha,\beta)$ denotes the set of probability measures on
$X\times X$ with marginals $\alpha$ and $\beta$.
\end{definition}

In our setting, $X=\mathcal{Z}_k$ is the embedding space equipped with a fixed
metric $d$, and $\mathcal{W}_1$ quantifies how well the empirical embedding
distribution $\hat P_{k,n}^Z$ approximates its population counterpart
$\tilde P_k^Z$.

\subsection{Network Decomposition and Embeddings}

\begin{definition}[Network decomposition]
\label{def:decomp}
At an intermediate layer $k$, the network is decomposed into an \emph{encoder}
$F_{\le k}:\mathcal{X}\to\mathcal{Z}_k$ mapping an input $x\in\mathcal{X}$ to an embedding
$z\in\mathcal{Z}_k$, and a \emph{tail map}
$F_k:\mathcal{Z}_k\to\mathbb{R}^C$ producing the final prediction.
The overall predictor can thus be written as
\[
F(x) = F_k\big(F_{\le k}(x)\big).
\]
\end{definition}

\begin{definition}[Empirical and population embedding distributions]
\label{def:emp-embed-dist}
Given $n$ i.i.d.\ samples $\{x_i\}_{i=1}^n \sim P_X$, the \emph{empirical embedding distribution} at layer $k$ is defined as
\[
\hat{\tilde P}_{k,n}^Z = \frac{1}{n}\sum_{i=1}^n \delta_{F_{\le k}(x_i)}.
\]
The \emph{population embedding distribution} $\tilde P_k^Z$ is the pushforward of $P_X$ under the embedding map,
\[
\tilde P_k^Z = \mathbb{E}_{x\sim P_X}\big[\delta_{F_{\le k}(x)}\big].
\]
\end{definition}

\noindent\textbf{Remark.}
In practice, we construct the empirical embedding distribution $\hat{\tilde P}_{k,n}^Z$ using embeddings from a held-out validation set, and approximate the population distribution $\tilde P_k^Z$ using embeddings from an independent test set. Since neither set is used for training, the resulting embeddings can be treated as approximately i.i.d.\ samples from their respective distributions. All representation-level quantities, including intrinsic dimension and Wasserstein distance, are computed based on these held-out embeddings rather than training-set embeddings, whose dependence on the learned parameters would violate the i.i.d.\ assumption.

\subsection{Lipschitz continuity}

\begin{definition}[Lipschitz map]
A function $f:(X,d_X)\to(Y,d_Y)$ is $L$-Lipschitz if
\[
d_Y\big(f(x),f(x')\big) \le L\, d_X(x,x')
\quad \text{for all } x,x'\in X.
\]
\end{definition}

The Lipschitz constant measures the sensitivity of the output to perturbations in the input and will be used to control how discrepancies in the embedding distribution propagate to the loss.

\subsection{Geometric Complexity and Wasserstein Convergence}

To characterize the convergence of empirical measures in Wasserstein distance, we adopt the geometric framework of \citet{weed2019sharp}.
The central insight of this framework is that Wasserstein convergence is governed
not by the ambient dimension of the embedding space, but by the intrinsic
geometric dimension of the underlying distribution, formalized through the
upper Wasserstein dimension.

\begin{definition}[Covering numbers and measure covering dimension]
Let $(X,d)$ be a metric space and $S\subseteq X$.
The $\varepsilon$--covering number of $S$ is
\[
\mathcal{N}_\varepsilon(S)
:= \min\left\{N : S \subseteq \bigcup_{i=1}^N B_i,\ \operatorname{diam}(B_i)\le \varepsilon \right\}.
\]
For a probability measure $\mu$ on $X$, the $(\varepsilon,\tau)$--covering number is
\[
\mathcal{N}_\varepsilon(\mu,\tau)
:= \inf\{\mathcal{N}_\varepsilon(S): \mu(S)\ge 1-\tau\},
\]
and the associated $(\varepsilon,\tau)$--dimension is
\[
d_\varepsilon(\mu,\tau)
:= \frac{\log \mathcal{N}_\varepsilon(\mu,\tau)}{-\log \varepsilon}.
\]
\end{definition}

\noindent\textbf{Remark.}
The quantity $d_\varepsilon(\mu,\tau)$ measures the effective geometric complexity
of the bulk of the distribution at scale $\varepsilon$, while allowing a
$\tau$--fraction of the probability mass to be ignored.

\begin{definition}[Upper Wasserstein dimension]
For a probability measure $\mu$ on $(X,d)$ and $p\ge 1$, the \emph{upper Wasserstein dimension} is defined as
\[
d_p^*(\mu)
:= 
\inf\left\{
s > 2p :
\limsup_{\varepsilon\to 0}
d_\varepsilon\!\left(\mu,\ \varepsilon^{sp/(s-2p)}\right)
\le s
\right\}.
\]
\end{definition}

\noindent\textbf{Remark.}
The tolerance parameter $\tau=\varepsilon^{sp/(s-2p)}$ allows a vanishing fraction
of the probability mass to be excluded as the scale $\varepsilon \to 0$, thereby
preventing pathological, high-complexity regions of negligible mass from
dominating the dimension estimate.
As shown in \citet{weed2019sharp}, the resulting quantity $d_p^*(\mu)$ precisely
characterizes the minimax convergence rates of empirical measures in Wasserstein
distance.

\paragraph{Wasserstein convergence rates.}
Let $\hat\mu_n$ denote the empirical distribution of $n$ i.i.d.\ samples from $\mu$.

\begin{theorem}[Wasserstein convergence governed by intrinsic dimension]
\label{thm:weedbach}
For any $p\in[1,\infty)$ and any $\varepsilon>0$, setting
$s = d_p^*(\mu)+\varepsilon$ yields
\[
\mathbb{E}\big[W_p(\mu,\hat\mu_n)\big]
\le C_{\varepsilon,p}\, n^{-1/s}.
\]
Since $\varepsilon$ may be chosen arbitrarily small, the convergence rate can be made arbitrarily close to $n^{-1/d_p^*(\mu)}$.
\end{theorem}

\noindent\textbf{Remark.}
Lower intrinsic dimension implies faster Wasserstein convergence of empirical
measures. In the context of representation learning, estimating an intrinsic
dimension proxy from embeddings therefore predicts how efficiently finite samples recover the population representation geometry.

\subsection{Risk and Bayes Predictor}

\begin{definition}[Population and empirical risk]
For a fixed predictor $F:\mathcal{X}\to\mathbb{R}^C$ and loss function $\ell$, the
\emph{population risk} is
\[
R(F) := \mathbb{E}_{(x,y)\sim P_{X,Y}}[\ell(F(x),y)],
\]
and the \emph{empirical risk} on $n$ i.i.d.\ samples $\{(x_i,y_i)\}_{i=1}^n$ is
\[
\hat R_n(F) := \frac{1}{n}\sum_{i=1}^n \ell(F(x_i),y_i).
\]
The quantity of interest in this work is the \emph{generalization gap}
$R(F)-\hat R_n(F)$ for a fixed, trained model $F$.
\end{definition}

\noindent\textbf{Remark.}
In experiments, the empirical risk is computed on the training set, while the
population risk is approximated using an independent test set.

\begin{definition}[Bayes predictor]
\label{def:bayes}
Given a population distribution $P_{X,Y}$ over inputs and outputs, the \emph{Bayes predictor} is defined as the conditional risk minimizer
\[
F^*(x) := \arg\min_{f}\;\mathbb{E}\big[\ell(f(X),Y)\mid X=x\big].
\]
It represents the population-optimal prediction achievable given access to the full data.
\end{definition}

In the context of a fixed network and an intermediate layer $k$, we denote by $F_k^*$ the Bayes predictor associated with the embedding $Z_k = F_{\le k}(X)$, corresponding to the population-optimal tail mapping from
$\mathcal{Z}_k$ to the output space.

\noindent\textbf{Remark.}
The Bayes predictor is a theoretical reference and is not assumed to be
computable or learned. It is introduced to separate approximation error due to information loss in the representation from estimation error due to finite samples. By replacing the discrete label $Y$ with the continuous target $F_k^*(Z_k)$, the loss becomes smooth with respect to the embedding distribution, enabling a distributional analysis. The resulting approximation error is explicitly accounted for in the generalization bounds.

\subsection{Standing Assumptions}
\label{ssec:assumptions}

We now state the regularity assumptions required for the analysis, together with their roles in the proofs.

\begin{assumption}[Measurability of embeddings]
\label{assump:measurable}
For each layer $k$, the embedding map $F_{\le k}:\mathcal{X}\to\mathcal{Z}_k$ is measurable, ensuring that the pushforward distribution $\tilde P_k^Z$ is well defined.
\end{assumption}

\begin{assumption}[Bounded support]
\label{assump:bounded}
Each embedding distribution $\tilde{P}_k^Z$ has bounded $\ell_1$-diameter:
\[
D_k := \sup_{z,z'\in \operatorname{supp}(\tilde{P}_k^Z)} \|z-z'\|_1 < \infty.
\]
The bounded diameter $D_k$ is used in Proposition \ref{prop:concentration} (Appendix \ref{appendix: Concentration of wasserstein}) to control the effect of a single-sample replacement when applying McDiarmid's inequality to the Wasserstein term $W_1(\tilde P_k^Z,\hat P_{k,n}^Z)$.
\end{assumption}

\begin{assumption}[Local Lipschitz continuity of tail and Bayes maps]
\label{assump:lipschitz_maps}
For each layer $k$, there exists an open neighborhood
$U_k\supseteq\operatorname{supp}(\tilde P_k^Z)$ on which both the network tail map
$F_k$ and the Bayes predictor $F_k^*$ are Lipschitz, with constants
\begin{align}
L_k(F) := \sup_{z \in U_k} \|\nabla F_k(z)\|_2, \notag\\
L_k(F^*) := \sup_{z \in U_k} \|\nabla F_k^*(z)\|_2. \notag
\end{align}
Here $\|\cdot\|_2$ denotes the induced 2-norm (spectral norm) of the Jacobian.
These constants control how perturbations in the embedding space propagate through the corresponding maps.
\end{assumption}

\begin{assumption}[Smooth loss with bounded gradient]
\label{assump:loss}
The loss $\ell:\mathbb{R}^C\times\mathbb{R}^C\to\mathbb{R}$ is continuously
differentiable in both arguments.
There exist constants $M_F,M_{F^*}<\infty$ such that
\[
\|\nabla_u \ell(u,v)\|_\infty \le M_F,
\qquad
\|\nabla_v \ell(u,v)\|_\infty \le M_{F^*}.
\]
This assumption is satisfied by commonly used smooth losses such as squared error and cross-entropy.
Non-smooth losses (e.g., hinge loss) can be accommodated via standard smoothing arguments.
These gradient bounds ensure that perturbations in the embeddings propagate in a controlled manner through the loss, making the subsequent analysis of embedding-induced errors feasible.
\end{assumption}

\begin{assumption}[Bounded loss of the Bayes predictor]
\label{assump:bayes_bounded}
Let $F^*$ denote the Bayes predictor defined in Definition~\ref{def:bayes}.
We assume that there exists a constant $M>0$ such that
\[
\ell\big(F^*(x),y\big) \le M
\quad \text{for all } (x,y) \sim P_{X,Y}.
\]
\end{assumption}

\section{Main Theoretical Results}
\label{sec:main_results}

\subsection{Dimension-dependent post-training generalization bound}
\label{ssec:main_theorem}

We first establish a bound on the \emph{generalization gap}
$R(F)-\hat R_n(F)$ for a fixed, trained predictor $F$.
The bound explicitly reveals how the intrinsic dimension of intermediate
embeddings and the sensitivity of downstream mappings jointly control
generalization.

\begin{theorem}[Dimension-dependent post-training generalization bound]
\label{thm:main_observed_mcd}
Assume Assumptions~\ref{assump:measurable}--\ref{assump:bayes_bounded}, and fix a confidence
level $\delta\in(0,1)$.
For each layer $k$, let $d_k := d_1^*(\tilde P_k^Z)$ denote the upper Wasserstein
dimension of the embedding distribution.
Then for any $\epsilon>0$, there exists a constant $C_k>0$ such that, for all
sufficiently large $n$,
\[
\mathbb{E}\!\left[\mathcal W_1(\tilde P_k^Z,\hat{\tilde P}_{k,n}^Z)\right]
\le C_k\, n^{-1/(d_k+\epsilon)}.
\]

Under these conditions, for any fixed predictor $F$, with probability at least
$1-\delta$,
\begin{align}
& R(F) \le \hat R_n(F) \notag\\
& + \min_{0\le k\le L} \Big\{ \bar L_k \Big(
C_k\, n^{-1/(d_k+\epsilon)} + D_k \sqrt{\tfrac{1}{2n}\log\tfrac{2(L+1)}{\delta}}
\Big) \ \notag\\
& + M_{F^*}\Big(
2\,\mathbb E\|Y-F_k^*(Z_k)\|_1 + \sqrt{\tfrac{2}{n}\log\tfrac{2(L+1)}{\delta}}
\Big)
\Big\},
\label{eq:main_bound_mcd}
\end{align}
where $D_k$ is the $\ell_1$-diameter of the embedding support and
\[
\bar L_k := L_k(F_k)\,M_F + L_k(F_k^*)\,M_{F^*}.
\]

Here $L_k(F_k)$ denotes the Lipschitz constant of the network tail map from layer $k$ to the output, and $L_k(F_k^*)$ denotes the Lipschitz constant of the Bayes predictor from layer $k$ to the output.  
$M_F$ and $M_{F^*}$ are the gradient bounds of the loss with respect to $F$ and $F^*$, as defined in Assumption~\ref{assump:loss}.  
Together, these constants quantify how perturbations in the embeddings propagate through the network and the loss, enabling the analysis of embedding-induced generalization error.
\end{theorem}

\noindent\textbf{Remarks.}
\begin{itemize}[leftmargin=*]
    \item \textbf{Dimension-controlled statistical error.}
    The dominant term $n^{-1/(d_k+\epsilon)}$ arises from the convergence rate of
    the empirical embedding distribution to its population counterpart in
    Wasserstein distance.
    Lower intrinsic dimension implies faster distributional convergence and thus
    smaller generalization gap for a fixed sample size.

    \item \textbf{Sensitivity amplification.}
    The factor $\bar L_k$ quantifies how discrepancies in the embedding
    distribution are amplified by downstream mappings and the loss.
    Even low-dimensional representations may generalize poorly if the predictor
    is highly sensitive to small embedding perturbations, highlighting the joint
    role of dimension and Lipschitz stability.

    \item \textbf{Bayes approximation (irreducible) error.}
    The term involving $\mathbb E\|Y-F_k^*(Z_k)\|_1$ captures the irreducible approximation error incurred when replacing discrete labels with the Bayes predictor. This component reflects intrinsic label noise or information loss in the representation and cannot be reduced by increasing sample size.

    \item \textbf{Layer-wise tradeoff.}
    Different layers induce different balances between intrinsic dimension and
    sensitivity.
    Early layers may exhibit higher dimension but lower sensitivity, while later
    layers are typically more compressed but potentially more sensitive.
    Minimizing over $k$ selects the representation that provides the tightest
    control of the generalization gap.
\end{itemize}

\subsection{Final-layer simplification}
\label{ssec:final_layer}

At the final layer, the representation coincides with the model output,
$Z_L = F(X)$.
In this case, the downstream mapping from embeddings to predictions is the
identity map, and hence introduces no additional architectural amplification.

\begin{corollary}[Final-layer bound]
\label{cor:final_layer_identity}
For the final embedding $Z_L$, the tail mapping is the identity and therefore
$L_L(F)=1$.
With probability at least $1-\delta$,
\begin{align}
R(F)
\;\le\;
\hat R_n(F)
&+ \big( M_F + L_L(F^*)\,M_{F^*} \big) \notag\\
& * \Big(
C_L\, n^{-1/(d_L+\epsilon)}
+ D_L \sqrt{\tfrac{1}{2n}\log\tfrac{2(L+1)}{\delta}}
\Big) \notag\\
&+ M_{F^*}\Big(
2\,\mathbb E\|Y-F_L^*(Z_L)\|_1 \notag\\
&+ \sqrt{\tfrac{2}{n}\log\tfrac{2(L+1)}{\delta}}
\Big).
\label{eq:corollary_final_layer_identity}
\end{align}
\end{corollary}

\noindent\textbf{Remark.}
At the final layer, architectural sensitivity disappears from the bound,
leaving a dependence only on:
(i) the intrinsic dimension $d_L$,
(ii) the embedding diameter $D_L$,
(iii) loss-dependent smoothness constants $(M_F,M_{F^*})$, and
(iv) the Bayes smoothness and irreducible error terms.
This simplification explains why final-layer intrinsic dimension often serves as
a strong empirical predictor of generalization: performance is governed primarily
by representation geometry and data-dependent smoothness rather than network
architecture.

\noindent\textbf{Summary.}  
The generalization error is determined by two main forces: the intrinsic dimension of embeddings (statistical efficiency) and Lipschitz sensitivity (stability to perturbations). Intermediate layers reflect both effects, requiring joint consideration of dimension and sensitivity. The final layer provides a simplified diagnostic where only dimension and distribution-dependent smoothness remain, clarifying why final-layer dimension has strong predictive power for generalization. The complete proof of Theorem \ref{thm:main_observed_mcd} is provided in Appendix \ref{Supplement: proof}.

\section{Experiments and Results}
\label{sec:exp_validation}

We empirically validate the theoretical predictions derived in
Section~\ref{sec:main_results}.
Each experiment is designed to isolate and test a specific component of the
generalization bound: (i) the dimension–controlled Wasserstein
convergence rate (Section \ref{sec:synthetic_wass}), (ii) the final-layer simplification that removes architectural sensitivity(Section \ref{sec:cross_arch} and \ref{sec:large_model}), and (iii) the joint role of embedding dimension and downstream Lipschitz sensitivity at intermediate layers (Section \ref{sec:causal_width}).

\subsection{Validation of Wasserstein Convergence Scaling}
\label{sec:synthetic_wass}

Theorem~\ref{thm:weedbach} predicts that the convergence rate of empirical to
population distributions in Wasserstein distance is governed by the intrinsic
dimension of the underlying distribution.
Before relating this behavior to generalization, we first ask whether this
scaling law holds for the complex, data-dependent embeddings produced by neural
networks.

\begin{figure}[htbp] 
  \centering
  \includegraphics[width=0.95\linewidth]{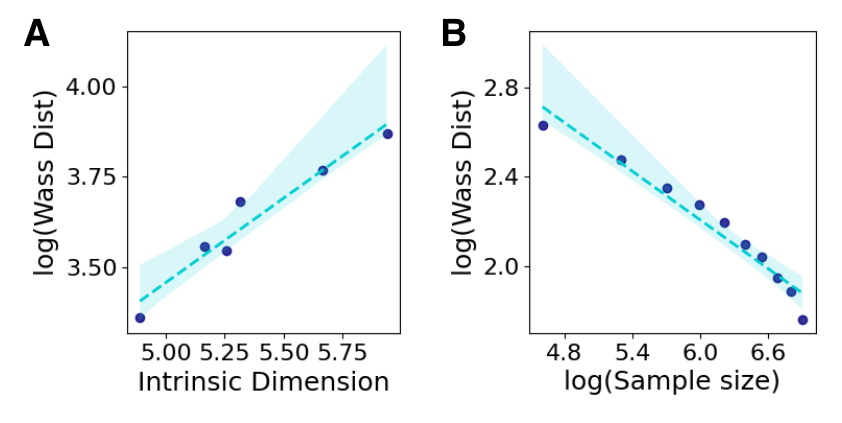}
  \caption{\textbf{Scaling of Wasserstein Convergence in Neural Network Embeddings.} 
  \textbf{(A)} With fixed sample size, log(Wasserstein distance) increases approximately linearly with embedding dimension.
  \textbf{(B)} With fixed embedding dimension, log(Wasserstein distance) decreases approximately linearly with log(sample size).}
  \label{fig:wass_scaling}
\end{figure}

We train a five-layer MLP autoencoder on MNIST and analyze how the
Wasserstein distance between empirical and population embedding distributions depends on both intrinsic dimension and sample size.
We examine two complementary perspectives.
First, we fix the sample size $n$ and study how the Wasserstein distance varies with the intrinsic dimension of the embedding.
Second, we fix the intrinsic dimension and evaluate how the Wasserstein distance scales with $n$ according to the predicted $n^{-1/(d+\epsilon)}$ rate. 
Sample sizes are varied from $n=100$ to $n=1500$.
For each configuration, intrinsic dimension is estimated using the MLE estimator of \citet{levina2004maximum}.

The results exhibit two consistent patterns.
For fixed $n$, the Wasserstein distance increases approximately exponentially with intrinsic dimension, as predicted by the dimension-dependent rate (Figure~\ref{fig:wass_scaling}A).
Conversely, for fixed $d$, the Wasserstein distance decreases approximately as a power law in $n$, closely matching the theoretical $n^{-1/(d+\epsilon)}$ scaling (Figure~\ref{fig:wass_scaling}B).
These findings confirm that intrinsic dimension accurately governs Wasserstein convergence even for learned neural network embeddings.
Additional experimental details are provided in Appendix~\ref{detail: synthetic_wass}.

\subsection{Analysis of Dimension, Wasserstein Distance and Generalization in Convolutional Networks}
\label{sec:cross_arch}

Corollary~\ref{cor:final_layer_identity} predicts that when analyzing the final
layer of a network, architectural sensitivity vanishes and the generalization
gap is primarily governed by the intrinsic dimension of the output embedding and
its Wasserstein convergence behavior.
We empirically test this prediction across architectures and datasets.

\begin{figure}[htbp]
  \centering
  \includegraphics[width=0.99\linewidth]{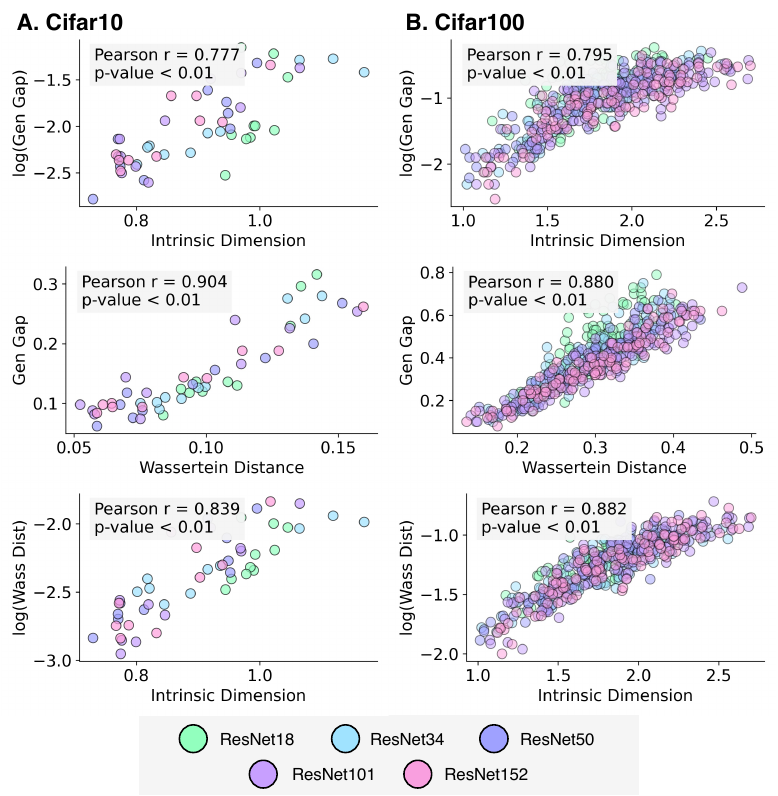}
  \caption{\textbf{Relationship Between Final-Layer Embedding Dimension, Wasserstein Distance and Generalization Error.}  
  We evaluate CIFAR-10 \textbf{(A)} and CIFAR-100 \textbf{(B)} and observe a significant correlation between final-layer embedding dimension, Wasserstein distance and generalization error.}
  \label{fig:sameDataDiffModel}
\end{figure}

We evaluate ResNet-18, 34, 50, 101, and 152 on CIFAR-10 and CIFAR-100.
For each trained model, we extract the final-layer embeddings and estimate their
intrinsic dimension.
We also compute the empirical Wasserstein distance between validation and test embedding distributions, which approximates the distributional discrepancy appearing in the final-layer bound.

To obtain finer-grained statistical resolution, we perform the analysis at the class level. Each ResNet model thus yields 10 data points on CIFAR-10 and 100 data points on CIFAR-100, allowing us to assess the relationship between embedding geometry and generalization gap across a broad range of conditions.

Figure~\ref{fig:sameDataDiffModel} shows a clear positive relationship between
final-layer intrinsic dimension, empirical Wasserstein distance, and the generalization gap.
Because architectural sensitivity is eliminated by construction at the final
layer, this result provides direct empirical support for the corollary and
demonstrates that representation geometry alone is sufficient to predict
generalization behavior across architectures.
Additional analyses and robustness checks are provided in
Appendix~\ref{detail: cross_arch} and Appendix \ref{supple: algorithm and hyperparameter}.

\subsection{Analysis of Dimension, Wasserstein Distance, and Generalization in Large Models}
\label{sec:large_model}




We further evaluate whether Corollary~\ref{cor:final_layer_identity} remains valid in the regime of large pretrained models and large-scale datasets. We perform the same analysis on pretrained vision models evaluated on ImageNet-1K and pretrained language models evaluated on MNLI \citep{williams2018broad}. Due to the lack of access to training data and training losses for such models, we estimate the Wasserstein distance using two disjoint subsets sampled from the dataset and characterize generalization performance using test accuracy. All architectural details are deferred to the Appendix \ref{appendix: largemodel}.

\begin{figure}[htbp]
  \centering
  \includegraphics[width=0.99\linewidth]{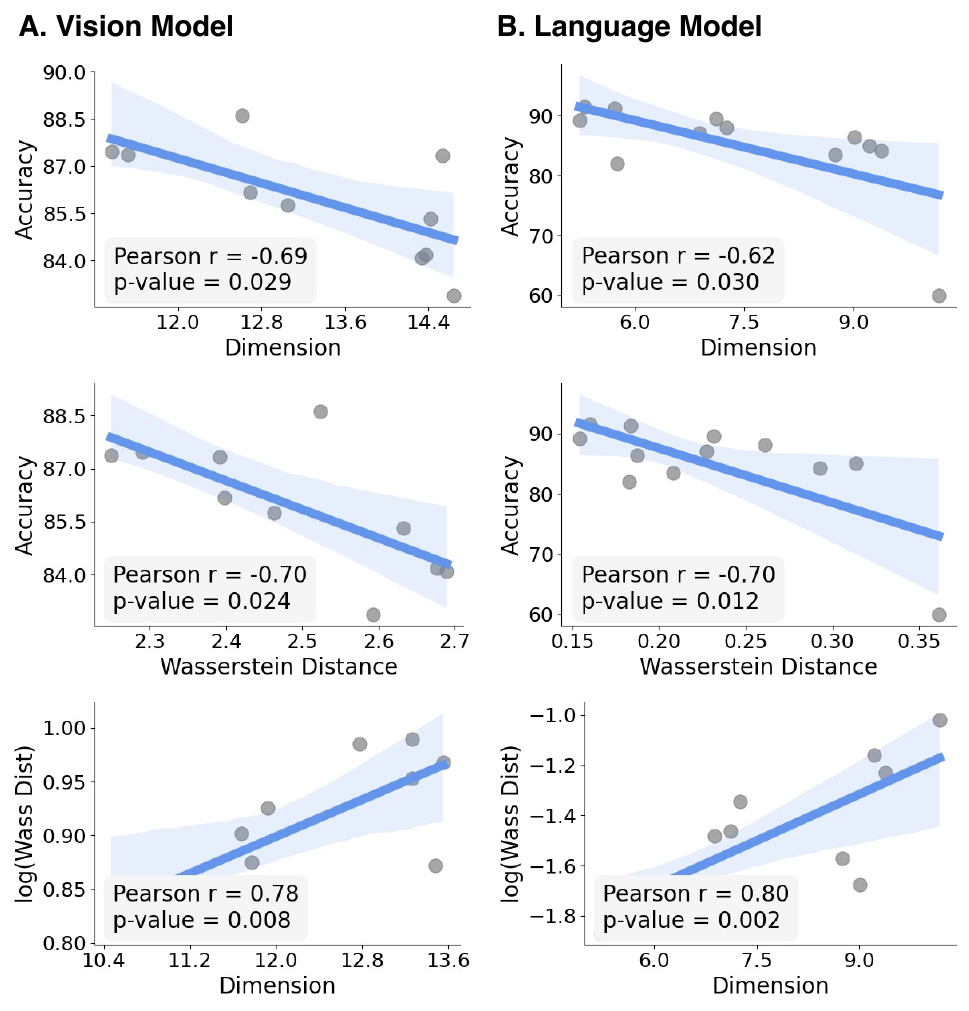}
    \caption{\textbf{Relationship between embedding geometry and generalization performance in large models.}
    (\textbf{A}) Vision models. (\textbf{B}) Language models. Across both modalities, embedding dimension and Wasserstein distance are strongly correlated with generalization performance, consistent with the theoretical predictions of our generalization bound.}
  \label{fig:LLMResults}
\end{figure}

As shown in Figure~\ref{fig:LLMResults}, we observe the same qualitative behavior as in previous section. Final-layer intrinsic dimension, Wasserstein distance and generalization performance remain strongly correlated for both vision and language models: lower-dimensional embeddings are associated with smaller Wasserstein distances and better test performance, and intrinsic dimension and Wasserstein distance are themselves tightly coupled.

These experiments demonstrate that the theoretical predictions remain valid for highly overparameterized models trained on massive datasets, showing that the connection between embedding geometry and generalization is not vacuous even as model size grows.

\subsection{Interventions on Network Width}
\label{sec:causal_width}

Theorem~\ref{thm:main_observed_mcd} predicts that at intermediate layers, generalization depends jointly on embedding intrinsic dimension and the Lipschitz sensitivity of the downstream mapping. To empirically test this interaction, we perform controlled architectural interventions that selectively modify one factor while tracking the other.

Directly computing the Lipschitz constant of deep networks is intractable in
general. We therefore consider fully connected MLPs with ReLU activations, for which the product of spectral norms of the weight matrices provides a computable upper bound on the network's Lipschitz constant \citep{bartlett2017spectrally}. We use this quantity as a sensitivity proxy and vary the width of a single intermediate layer to study its effect on both embedding geometry and sensitivity.

\begin{figure}[htbp] 
  \centering
  \includegraphics[width=0.85\linewidth]{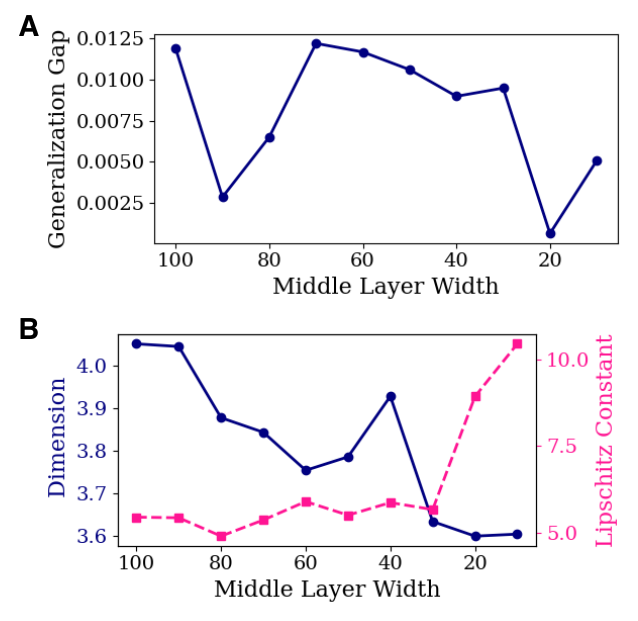}
    \caption{\textbf{Effect of Network Width on Embedding Dimension and Generalization.} 
    \textbf{(A)} Reducing the width of the third layer does not lead to a consistent decrease in generalization error. 
    \textbf{(B)} As the layer width decreases, the embedding dimension gradually decreases, but the network's Lipschitz constant increases, offsetting potential gains.}
  \label{fig:dimLipsRelation}
\end{figure}

Specifically, we train a six-layer MLP on CIFAR-10 and vary the width of the third
layer from 100 to 10. Figure~\ref{fig:dimLipsRelation} shows that while narrowing the layer consistently reduces the intrinsic dimension of the embedding, it also increases the network's Lipschitz constant, particularly at small widths. As a result, the generalization gap does not improve monotonically.

These results provide causal evidence for the tradeoff predicted by our theory: generalization is jointly governed by representation geometry and the sensitivity of the downstream mapping. While narrowing the network enforces lower-dimensional embeddings, it can simultaneously increase the Lipschitz constant of the learned mapping, thereby offsetting potential gains in generalization. This explains why architectural compression alone does not guarantee improved generalization and highlights the necessity of considering both embedding dimensionality and network stability in overparameterized models. Additional details are provided in Appendix~\ref{detail: causal_width}.

\section{Discussion and Conclusion}
\label{sec:discussion}

Understanding why deep networks generalize despite massive overparameterization remains a central challenge.  
We advance a \emph{representation-centric view}, showing that generalization error can be related to two measurable properties: the intrinsic dimension $d_k$ of embeddings and a sensitivity term $\bar{L}_k$ that quantifies how embedding perturbations propagate through the network.  
These quantities integrate model structure and data distribution, offering post-hoc diagnostics beyond classical capacity-based bounds. 

In Appendix~\ref{appendix: layerwise}, we analyze layer-wise embeddings in ResNet-154, computing correlations among intrinsic dimension, Wasserstein distance, and generalization performance. We find that intrinsic dimension and Wasserstein distance are already strongly correlated in early layers, while the correlations between dimension and generalization, as well as between Wasserstein distance and generalization, progressively increase with network depth. This confirms that embedding geometry plays an increasingly important role in explaining generalization in deeper layers.

\paragraph{Limitations.}  
Our bound contains constants that may be loose, yet experiments show that variations in embedding dimension effectively track generalization error, confirming practical relevance. Assumptions such as Lipschitz continuity of the Bayes predictor are necessary to relate a specific layer's embedding to the output; relaxing these is an important future direction. Estimating the Lipschitz constants remains challenging in practice, and developing reliable and scalable methods for bounding or estimating Lipschitz constants in deep networks is an important open problem.

\paragraph{Conclusion.}  
Focusing on embedding geometry rather than network parameters, we identify intrinsic dimension and Lipschitz Constant of network as core drivers of generalization. This framework provides both theoretical insight and practical tools for analyzing  deep networks.

\section*{Impact Statement}

This work aims to advance the understanding of generalization in deep learning by analyzing how embedding geometry and network sensitivity influence performance. Our findings primarily provide theoretical and practical insights for designing and analyzing neural networks. We do not anticipate any direct negative societal consequences from this research. While advances in machine learning can have broad societal implications, the present study is foundational in nature, and any potential downstream applications will depend on how these insights are used in specific contexts.

\bibliography{example_paper}

@article{ansuini2019intrinsic,
  title={Intrinsic dimension of data representations in deep neural networks},
  author={Ansuini, Alessio and Laio, Alessandro and Macke, Jakob H and Zoccolan, Davide},
  journal={Advances in Neural Information Processing Systems},
  volume={32},
  year={2019}
}

@article{levina2004maximum,
  title={Maximum likelihood estimation of intrinsic dimension},
  author={Levina, Elizaveta and Bickel, Peter},
  journal={Advances in neural information processing systems},
  volume={17},
  year={2004}
}

@article{pope2021intrinsic,
  title={The intrinsic dimension of images and its impact on learning},
  author={Pope, Phillip and Zhu, Chen and Abdelkader, Ahmed and Goldblum, Micah and Goldstein, Tom},
  journal={arXiv preprint arXiv:2104.08894},
  year={2021}
}

@article{weed2019sharp,
  title={Sharp asymptotic and finite-sample rates of convergence of empirical measures in Wasserstein distance},
  author={Weed, Jonathan and Bach, Francis},
  journal={Bernoulli},
  volume={25},
  number={4A},
  pages={2620--2648},
  year={2019},
  publisher={JSTOR}
}

@book{villani2009optimal,
  title     = {Optimal Transport: Old and New},
  author    = {Villani, C{\'e}dric},
  series    = {Grundlehren der mathematischen Wissenschaften},
  volume    = {338},
  publisher = {Springer},
  year      = {2009},
  isbn      = {9783540710493},
  doi       = {10.1007/978-3-540-71050-9}
}

@misc{sain1996nature,
  title={The nature of statistical learning theory},
  author={Sain, Stephan R},
  year={1996},
  publisher={Taylor \& Francis}
}

@article{bartlett2002rademacher,
  title={Rademacher and gaussian complexities: Risk bounds and structural results},
  author={Bartlett, Peter L and Mendelson, Shahar},
  journal={Journal of machine learning research},
  volume={3},
  number={Nov},
  pages={463--482},
  year={2002}
}

@inproceedings{arora2018stronger,
  title={Stronger generalization bounds for deep nets via a compression approach},
  author={Arora, Sanjeev and Ge, Rong and Neyshabur, Behnam and Zhang, Yi},
  booktitle={International conference on machine learning},
  pages={254--263},
  year={2018},
  organization={PMLR}
}

@article{amsaleg2018extreme,
  title={Extreme-value-theoretic estimation of local intrinsic dimensionality},
  author={Amsaleg, Laurent and Chelly, Oussama and Furon, Teddy and Girard, St{\'e}phane and Houle, Michael E and Kawarabayashi, Ken-ichi and Nett, Michael},
  journal={Data Mining and Knowledge Discovery},
  volume={32},
  number={6},
  pages={1768--1805},
  year={2018},
  publisher={Springer}
}

@inproceedings{amsaleg2019intrinsic,
  title={Intrinsic dimensionality estimation within tight localities},
  author={Amsaleg, Laurent and Chelly, Oussama and Houle, Michael E and Kawarabayashi, Ken-Ichi and Radovanovi{\'c}, Milo{\v{s}} and Treeratanajaru, Weeris},
  booktitle={Proceedings of the 2019 SIAM international conference on data mining},
  pages={181--189},
  year={2019},
  organization={SIAM}
}

@article{davies2009cluster,
  title={A cluster separation measure},
  author={Davies, David L and Bouldin, Donald W},
  journal={IEEE transactions on pattern analysis and machine intelligence},
  number={2},
  pages={224--227},
  year={2009},
  publisher={Ieee}
}

@article{dyballa2024separability,
  title={A separability-based approach to quantifying generalization: which layer is best?},
  author={Dyballa, Luciano and Gerritz, Evan and Zucker, Steven W},
  journal={arXiv preprint arXiv:2405.01524},
  year={2024}
}

@inproceedings{belcher2020generalisation,
  title={Generalisation and the geometry of class separability},
  author={Belcher, Dominic and Prugel-Bennett, Adam and Dasmahapatra, Srinandan},
  booktitle={NeurIPS 2020 Workshop: Deep Learning through Information Geometry},
  year={2020}
}

@article{bartlett2017spectrally,
  title={Spectrally-normalized margin bounds for neural networks},
  author={Bartlett, Peter L and Foster, Dylan J and Telgarsky, Matus J},
  journal={Advances in neural information processing systems},
  volume={30},
  year={2017}
}

@article{vapnik1994measuring,
  title={Measuring the VC-dimension of a learning machine},
  author={Vapnik, Vladimir and Levin, Esther and Le Cun, Yann},
  journal={Neural computation},
  volume={6},
  number={5},
  pages={851--876},
  year={1994},
  publisher={MIT Press}
}

@article{sontag1998vc,
  title={VC dimension of neural networks},
  author={Sontag, Eduardo D and others},
  journal={NATO ASI Series F Computer and Systems Sciences},
  volume={168},
  pages={69--96},
  year={1998},
  publisher={Springer Verlag}
}

@article{truong2022rademacher,
  title={On rademacher complexity-based generalization bounds for deep learning},
  author={Truong, Lan V},
  journal={arXiv preprint arXiv:2208.04284},
  year={2022}
}

@article{hellstrom2025generalization,
  title={Generalization bounds: Perspectives from information theory and PAC-Bayes},
  author={Hellstr{\"o}m, Fredrik and Durisi, Giuseppe and Guedj, Benjamin and Raginsky, Maxim and others},
  journal={Foundations and Trends{\textregistered} in Machine Learning},
  volume={18},
  number={1},
  pages={1--223},
  year={2025},
  publisher={Now Publishers, Inc.}
}

@article{lotfi2022pac,
  title={PAC-Bayes compression bounds so tight that they can explain generalization},
  author={Lotfi, Sanae and Finzi, Marc and Kapoor, Sanyam and Potapczynski, Andres and Goldblum, Micah and Wilson, Andrew G},
  journal={Advances in Neural Information Processing Systems},
  volume={35},
  pages={31459--31473},
  year={2022}
}

@inproceedings{neyshabur2015norm,
  title={Norm-based capacity control in neural networks},
  author={Neyshabur, Behnam and Tomioka, Ryota and Srebro, Nathan},
  booktitle={Conference on learning theory},
  pages={1376--1401},
  year={2015},
  organization={PMLR}
}

@article{neyshabur2017pac,
  title={A pac-bayesian approach to spectrally-normalized margin bounds for neural networks},
  author={Neyshabur, Behnam and Bhojanapalli, Srinadh and Srebro, Nathan},
  journal={arXiv preprint arXiv:1707.09564},
  year={2017}
}

@inproceedings{williams2018broad,
  title={A broad-coverage challenge corpus for sentence understanding through inference},
  author={Williams, Adina and Nangia, Nikita and Bowman, Samuel},
  booktitle={Proceedings of the 2018 conference of the North American chapter of the association for computational linguistics: human language technologies, volume 1 (long papers)},
  pages={1112--1122},
  year={2018}
}
\bibliographystyle{icml2026}

\newpage
\appendix
\onecolumn

\section*{Overview of Supplementary Materials}

This supplementary material provides detailed derivations, proofs, and experimental analyses that complement the main text. 
Its purpose is to give a complete and self-contained presentation of our theoretical results, as well as extensive experimental validation of the proposed framework. 
The content can be broadly organized into two main parts: theoretical analysis and experimental details.

\paragraph{1. Theoretical Analysis (Sections~\ref{Supplement: proof}).}  
We derive a high-probability, dimension-dependent generalization bound for neural networks, emphasizing the role of intermediate embedding geometry and network Lipschitz properties.  
Key steps include:
\begin{itemize}[leftmargin=*]
    \item \textbf{Notation and preliminaries (Sectio \ref{app:preliminaries}):} Introduce all key symbols and collect standard results on optimal transport and Wasserstein distances.  
    \item \textbf{Risk decomposition via Bayes surrogates (Section~\ref{app:decomposition}):} Split the generalization gap into (A) approximation, (B) oracle statistical, and (C) empirical model terms, enabling Lipschitz-based analysis.  
    \item \textbf{Bounding each component (Section~\ref{app:control}):} Control each term using deterministic Lipschitz arguments and concentration inequalities, producing high-probability bounds that scale with embedding dimension and sample size.  
    \item \textbf{Recovering network effects (Section~\ref{app:network_effects}):} Decompose the oracle Lipschitz constant to separate controllable network contributions from intrinsic distribution-dependent effects, yielding interpretable insights on architecture influence.
\end{itemize}

\paragraph{2. Experimental Analyses (Sections~\ref{detail: Experiments} and \ref{supple: algorithm and hyperparameter}).}  
We provide comprehensive experimental validations to illustrate the theoretical insights and investigate practical aspects of embeddings and network design.  
The main experimental directions are:
\begin{itemize}[leftmargin=*]
    \item \textbf{Synthetic embedding experiments (Section~\ref{detail: synthetic_wass}):} Study how Wasserstein distances depend on intrinsic dimension and sample size using MNIST autoencoders.  
    \item \textbf{Cross-architecture analysis (Section~\ref{detail: cross_arch}):} Analyze embedding distributions and class-wise generalization gaps across multiple ResNet architectures on CIFAR-10/100.  
    \item \textbf{Causal width experiments (Section~\ref{detail: causal_width}):} Explore how hidden layer width affects embedding dimensionality, network Lipschitz constants, and generalization in MNIST MLPs.  
    \item \textbf{Dimensionality estimation and hyperparameter sensitivity (Section~\ref{supple: algorithm and hyperparameter}):} Assess the influence of hyperparameters and estimation methods on intrinsic dimension measurements.
\end{itemize}

Together, these theoretical and empirical components provide a comprehensive understanding of the interplay between embedding geometry, network design, and generalization, complementing and extending the results presented in the main paper.

\section{Supplement of Theoretical Results}
\label{Supplement: proof}

Before presenting the detailed proofs, we first summarize the key notation used throughout this appendix and the main paper.  
This notation table serves as a convenient reference to improve clarity and readability.

\paragraph{Notation summary.}  
Key symbols used throughout the paper.

\begin{longtable}{ll}
\caption{Notation summary for key symbols in the paper.} \\
\toprule
\textbf{Symbol} & \textbf{Meaning} \\
\midrule
\endfirsthead

\toprule
\textbf{Symbol} & \textbf{Meaning} \\
\midrule
\endhead

$\tilde P_k^Z$ & Population embedding distribution at layer $k$. \\
$\hat{\tilde P}_{k,n}^Z$ & Empirical embedding distribution from $n$ samples. \\
$\hat\mu_n$ & Empirical measure of $n$ i.i.d.\ samples. \\
\midrule
$R(F)$ & Population risk of predictor $F$. \\
$\hat R_n(F)$ & Empirical risk on validation set. \\
$\mathrm{gen}(F)$ & Generalization gap $R(F)-\hat R_n(F)$. \\
\midrule
$F_{\le k}$ & Encoder mapping input $x$ to embedding $z$ at layer $k$. \\
$F_k$ & Tail map from layer-$k$ embedding $z$ to output. \\
$F(x)$ & Overall predictor $F_k(F_{\le k}(x))$. \\
$F_k^*$ & Bayes predictor from layer-$k$ embedding $z$ to output. \\
\midrule
$d_k$ & Intrinsic dimension of $\tilde P_k^Z$. \\
$D_k$ & $\ell_1$-diameter of support of $\tilde P_k^Z$. \\
$\mathcal{W}_1(\cdot,\cdot)$ & 1-Wasserstein distance. \\
\midrule
$L_k(F)$ & Lipschitz constant of network tail from layer $k$ to output. \\
$L_k(F^*)$ & Lipschitz constant of Bayes predictor from layer $k$ to output. \\
$M_F$ & Bound on loss gradient wrt network output. \\
$M_{F^*}$ & Bound on loss gradient wrt Bayes predictor output. \\
\midrule
$\ell$ & Loss function (e.g., squared loss, cross-entropy). \\
$B_\ell$ & Uniform bound on loss values. \\
\bottomrule
\end{longtable}

\paragraph{Roadmap of the appendix.}  
This appendix provides a complete derivation of the dimension-dependent generalization bound stated in Theorem~\ref{thm:main_observed_mcd}.  
The proof is organized into four main steps:

\begin{enumerate}[leftmargin=*]
    \item \textbf{Preliminaries (Subsection~\ref{app:preliminaries}):}  
    We collect standard technical tools used throughout the proofs, including optimal transport results and Wasserstein bounds for Lipschitz functions.  

    \item \textbf{Risk decomposition via Bayes surrogates (Subsection~\ref{app:decomposition}):}  
    In classification settings, labels are discrete, so the observed loss is non-differentiable with respect to embeddings.  
    We introduce layer-wise Bayes predictors as continuous surrogates, leading to a decomposition of the generalization error into three terms: (A) approximation gap, (B) oracle statistical gap, and (C) empirical model gap.

    \item \textbf{Controlling the decomposed terms (Subsection~\ref{app:control}):}  
    Each term is bounded explicitly.  
    (A) and (C) are controlled by irreducible label noise, while (B) is controlled via the 1-Wasserstein distance between empirical and population embeddings combined with the oracle loss Lipschitz constant.  
    Concentration inequalities yield high-probability bounds scaling with embedding dimension and sample size.

    \item \textbf{Recovering network effects (Subsection~\ref{app:network_effects}):}  
    The oracle Lipschitz constant \(L_k(g)\) is decomposed as
    \[
        L_k(g) \le L_k(F)\,M_F + L_k(F^*)\,M_{F^*},
    \]
    separating controllable network-dependent and intrinsic Bayes predictor contributions.  
    Substituting this into the previous bounds connects embedding geometry, statistical concentration, and network design.
\end{enumerate}

Overall, these steps provide a clear, high-probability generalization bound that disentangles statistical, architectural, and label-noise contributions.

\subsection{Preliminaries: Useful Lemmas and Theorems}
\label{app:preliminaries}

In this subsection we collect several standard results that will be used throughout the proofs. 
They are presented here to avoid interruptions in the main arguments later.

\subsubsection{Existence of Optimal Transport Plan}

\begin{theorem}[Existence of Optimal Transport Plan {\citep{villani2009optimal}}]
\label{thm:exist_opt_plan}
Let $(\mathcal{X},\mu)$ and $(\mathcal{Y},\nu)$ be Polish probability spaces, and let 
\[
c : \mathcal{X} \times \mathcal{Y} \;\longrightarrow\; \mathbb{R} \cup \{+\infty\}
\]
be a lower semicontinuous cost function.  
Then there exists a coupling $\gamma^* \in \Pi(\mu,\nu)$ that minimizes the expected cost:
\[
\int_{\mathcal{X}\times\mathcal{Y}} c(x,y)\,d\gamma^*(x,y)
= \inf_{\gamma\in\Pi(\mu,\nu)}\int_{\mathcal{X}\times\mathcal{Y}} c(x,y)\,d\gamma(x,y).
\]
In particular, for the 1-Wasserstein cost $c(z,z')=\|z-z'\|_1$ on a Polish space there exists an optimal coupling attaining $\mathcal W_1$.
\end{theorem}

\subsubsection{Wasserstein Bound for Lipschitz Functions}

\begin{lemma}[Expectation difference controlled by $W_1$]
\label{lem:W1_bound}
Let $(\mathbb R^{d},\|\cdot\|_1)$ be equipped with the $\ell_1$ metric, and let
$\mu,\nu$ be probability measures with finite first moments.
If $h:\mathbb R^{d}\to\mathbb R$ is $L_h$-Lipschitz with respect to $\ell_1$, i.e.,
\[
|h(z)-h(z')|\le L_h\|z-z'\|_1 \quad \forall z,z',
\]
then
\[
\Big|\int h\,d\mu - \int h\,d\nu\Big| \;\le\; L_h \,\mathcal W_1(\mu,\nu).
\]
\end{lemma}

\begin{proof}
This follows directly from the Kantorovich--Rubinstein dual representation of $\mathcal W_1$ on Polish metric spaces.

By definition of $\mathcal W_1$ and for any coupling $\pi\in\Pi(\mu,\nu)$,
\[
\int h\,d\mu - \int h\,d\nu
= \iint \big(h(z)-h(z')\big)\,d\pi(z,z')
\le \iint L_h \|z-z'\|_1 \, d\pi(z,z').
\]
Taking infimum over all couplings $\pi$ gives the claim. The absolute value follows by symmetry (swapping $\mu,\nu$).
\end{proof}

\subsection{Risk Decomposition via Bayes Surrogates}
\label{app:decomposition}

\paragraph{Motivation.}
In classification, the label $Y$ is discrete, so the observed loss 
$\ell(F(X),Y)$ is not differentiable with respect to embeddings $Z_k$. 
This obstructs a direct Lipschitz-based analysis of the risk, 
which is central to our approach. 
To address this, we introduce at each layer $k$ the \emph{Bayes predictor} $F_k^*(Z_k)$, 
a continuous surrogate for the discrete label. 
Replacing $Y$ with $F_k^*(Z_k)$ yields the \emph{oracle loss}, 
which is differentiable in $Z_k$ and hence amenable to Lipschitz/Wasserstein analysis. 
The cost of this replacement is an additional error term capturing the 
mismatch between observed and oracle risks. 
This term corresponds to irreducible label randomness and will be explicitly controlled. 

\begin{definition}[Observed and oracle risks]
Let $\ell:\mathbb R^C\times\mathbb R^C\to\mathbb R$ be a measurable loss.  
At the input layer, the observed risks are
\[
R_0^{\mathrm{obs}} := \mathbb E_{(X,Y)\sim\mathcal D}[\ell(F(X),Y)], 
\qquad
\hat R_{0,n}^{\mathrm{obs}} := \frac{1}{n}\sum_{i=1}^n \ell(F(x_i),y_i).
\]
At layer $k$, the oracle loss is defined as
\[
g_k(z) := \ell(F_k(z),F_k^*(z)),
\]
with population and empirical oracle risks
\[
R_k^{\mathrm{oracle}} := \mathbb E_{Z_k\sim\tilde P_k^Z}[g_k(Z_k)], 
\qquad
\hat R_{k,n}^{\mathrm{oracle}} := \frac{1}{n}\sum_{i=1}^n g_k(z_{k,i}).
\]
\end{definition}

The following identity is purely algebraic and holds deterministically for any fixed dataset and predictor.

\begin{proposition}[Risk decomposition]
\label{prop:decomp_app}
For any predictor $F$ and any intermediate layer $k$,
\[
R_0^{\mathrm{obs}} - \hat R_{0,n}^{\mathrm{obs}}
= (R_0^{\mathrm{obs}} - R_k^{\mathrm{oracle}})
+ (R_k^{\mathrm{oracle}} - \hat R_{k,n}^{\mathrm{oracle}})
+ (\hat R_{k,n}^{\mathrm{oracle}} - \hat R_{0,n}^{\mathrm{obs}}).
\]
\end{proposition}

\paragraph{Interpretation.}
The decomposition separates the generalization gap into three terms:
\begin{itemize}[leftmargin=*]
    \item \textbf{Approximation gap:} 
    $R_0^{\mathrm{obs}} - R_k^{\mathrm{oracle}}$ 
    measures the loss of information when replacing discrete labels by the Bayes predictor at layer $k$. 
    \item \textbf{Oracle statistical gap:} 
    $R_k^{\mathrm{oracle}} - \hat R_{k,n}^{\mathrm{oracle}}$ 
    is the population-to-sample deviation of the oracle loss, 
    the term to be controlled via Lipschitz continuity and Wasserstein concentration. 
    \item \textbf{Empirical model gap:} 
    $\hat R_{k,n}^{\mathrm{oracle}} - \hat R_{0,n}^{\mathrm{obs}}$ 
    quantifies how network predictions differ from Bayes-optimal predictions under the empirical distribution. 
\end{itemize}

\paragraph{Summary.}
The observed generalization error is thus expressed as an oracle component 
(amenable to Lipschitz/Wasserstein analysis) plus two additional error terms 
that capture irreducible label noise and model approximation. 
This motivates analyzing the Lipschitz constant of the oracle loss $g_k(z)$, 
which we do next.

\subsection{Lipschitz Constant of the Layer-Wise Loss}
\label{app:lipschitz}

Having introduced the oracle loss $g_k(z)=\ell(F_k(z),F_k^*(z))$, 
we now analyze its Lipschitz continuity with respect to the embedding $z$. 
This is possible because both arguments of $g_k$ are continuous functions of $z$.

\paragraph{Gradient and Lipschitz bound.}
\begin{lemma}
\label{lem:grad_bound_app}
Suppose Assumptions \ref{assump:measurable}--\ref{assump:loss} hold.
Then for all $z\in U_k$,
\[
\nabla g_k(z) 
= \nabla F_{k}(z)^\top \,\partial_F \ell(F_{k}(z),F_k^*(z))
+ \nabla F_k^*(z)^\top \,\partial_{F^*} \ell(F_{k}(z),F_k^*(z)).
\]
Hence
\[
\|\nabla g_k(z)\|_\infty
\;\le\;
\|\nabla F_{k}(z)\|_{\mathrm{op}} \,\|\partial_F \ell\|_\infty
+ \|\nabla F_k^*(z)\|_{\mathrm{op}} \,\|\partial_{F^*}\ell\|_\infty.
\]
If $\|\partial_F \ell\|_\infty\le M_F$, $\|\partial_{F^*}\ell\|_\infty\le M_{F^*}$,
and the Jacobians satisfy $\|\nabla F_{k}(z)\|_{\mathrm{op}}\le L_k(F)$,
$\|\nabla F_k^*(z)\|_{\mathrm{op}}\le L_k(F^*)$, then
\[
L_k(g):=\sup_{z\in U_k}\|\nabla g_k(z)\|_\infty
\;\le\; L_k(F)\,M_F + L_k(F^*)\,M_{F^*}.
\]
\end{lemma}

\begin{proof}
The chain rule gives the gradient expression.
Applying $\|A^\top v\|_\infty \le \|A\|_{\mathrm{op}}\|v\|_\infty$
and substituting the uniform bounds yields the claim.
\end{proof}

\noindent\textbf{Remark.}  
The bound cleanly separates two contributions:
(i) the network-dependent Lipschitz constant $L_k(F)$, controllable by architecture or regularization,
and (ii) the distribution-dependent Lipschitz constant $L_k(F^*)$, 
reflecting the inherent complexity of the Bayes predictor. 
Thus the oracle loss Lipschitz constant factors into 
a controllable and an uncontrollable component, 
which will play distinct roles in the final generalization bound.

\subsection{Controlling the Decomposed Terms}
\label{app:control}

\paragraph{Overview of the approach.}
Proposition~\ref{prop:decomp_app} decomposes the generalization gap into three terms:
\[
\underbrace{R_0^{\mathrm{obs}} - R_k^{\mathrm{oracle}}}_{(A)\,\text{approximation gap}}
,\quad
\underbrace{R_k^{\mathrm{oracle}} - \hat R_{k,n}^{\mathrm{oracle}}}_{(B)\,\text{oracle statistical gap}}
,\quad
\underbrace{\hat R_{k,n}^{\mathrm{oracle}} - \hat R_{0,n}^{\mathrm{obs}}}_{(C)\,\text{empirical model gap}}.
\]
We now control these terms separately:
\begin{itemize}
  \item (A) measures the error incurred by replacing labels $Y$ with the Bayes surrogate $F_k^*(Z_k)$.  
  \item (B) measures the statistical deviation between population and empirical distributions of embeddings, for the oracle loss.  
  \item (C) measures the discrepancy between network predictions and Bayes-optimal predictions under the empirical distribution.  
\end{itemize}
Each of (A), (B), (C) will be treated in turn.

\subsubsection{Bounding the approximation gap (A).}

\begin{lemma}[Control of approximation gap]
\label{lem:gap_A}
Assume the loss $\ell:\mathbb R^C\times\mathbb R^C\to\mathbb R$ is Lipschitz
in its second argument with constant $M_{F^*}$ (Assumption~\ref{assump:loss}).
Then for any predictor $F$ and any layer $k$,
\[
\big|R_0^{\mathrm{obs}} - R_k^{\mathrm{oracle}}\big|
\;\le\; M_{F^*}\; \mathbb E_{Z\sim\tilde P_k^Z}\big[\|Y - F_k^*(Z)\|_1\big].
\]
\end{lemma}

\begin{proof}
For any sample $(x,y)$ with embedding $z=h_{\le k}(x)$,
\[
\big|\ell(F(x),y) - \ell(F_k(z),F_k^*(z))\big|
\le M_{F^*}\,\|y - F_k^*(z)\|_1,
\]
by Lipschitz continuity of $\ell$ in the second argument.
Taking expectation over $(X,Y)\sim\mathcal D$ yields the result.
\end{proof}

\subsubsection{Bounding the oracle statistical gap (B).}

\begin{lemma}[Oracle risk controlled by $W_1$]
\label{lem:gap_B}
For any predictor $F\in\mathcal F$ and any layer $k$,
\[
\big|R_k^{\mathrm{oracle}} - \hat R_{k,n}^{\mathrm{oracle}}\big|
\;\le\; L_k(g)\; \mathcal W_1\big(\tilde P_k^Z,\hat{\tilde P}_{k,n}^Z\big).
\]
where \(L_k(g)\) is the Lipschitz constant of \(g_k(z)=\ell(F_k(z),F_k^*(z))\)
with respect to the $\ell_1$-metric, as given in Lemma~\ref{lem:grad_bound_app}.
\end{lemma}

\begin{proof}
By Kantorovich-Rubinstein duality, for any $L$-Lipschitz function $f$,
\[
\Big|\int f\,d\mu - \int f\,d\nu\Big|\;\le\; L\,W_1(\mu,\nu).
\]
Applying this with $f=g_k$, $\mu=\tilde P_k^Z$, and $\nu=\hat{\tilde P}_{k,n}^Z$,
and recalling that $g_k$ has Lipschitz constant $L_k(g)$,
gives the desired bound.
\end{proof}

\subsubsection{Bounding the empirical model gap (C).}

\begin{lemma}[Control of empirical model gap]
\label{lem:gap_C}
Under the same assumptions as Lemma~\ref{lem:gap_A}, for any predictor $F$ and any layer $k$,
\[
\big|\hat R_{k,n}^{\mathrm{oracle}} - \hat R_{0,n}^{\mathrm{obs}}\big|
\;\le\; M_{F^*}\; \frac{1}{n}\sum_{i=1}^n \|y_i - F_k^*(z_{k,i})\|_1.
\]
\end{lemma}

\begin{proof}
For each validation sample $(x_i,y_i)$ with embedding $z_{k,i}=h_{\le k}(x_i)$,
\[
\big|\ell(F(x_i),y_i) - \ell(F_k(z_{k,i}),F_k^*(z_{k,i}))\big|
\le M_{F^*}\,\|y_i - F_k^*(z_{k,i})\|_1.
\]
Averaging over $i=1,\dots,n$ yields the result.
\end{proof}

\subsubsection{Concentration of \(T_k:=\mathcal W_1(\tilde P_k^Z,\hat{\tilde P}_{k,n}^Z)\) and of the empirical noise average}
\label{appendix: Concentration of wasserstein}

\paragraph{Motivation.}
The deterministic decomposition in Proposition~\ref{prop:decomp_app} reduces the
generalization gap to three terms. Among them, two depend explicitly on the
randomness of the empirical sample:
\begin{itemize}
  \item the Wasserstein distance 
  \(T_k=\mathcal W_1(\tilde P_k^Z,\hat{\tilde P}_{k,n}^Z)\), 
  which controls the oracle statistical gap (B);
  \item the empirical noise average 
  \(\bar u^{(k)} = \tfrac{1}{n}\sum_{i=1}^n \|y_i-F_k^*(z_{k,i})\|_1\), 
  which appears in the empirical model gap (C).
\end{itemize}
To obtain a high-probability generalization bound, it is therefore crucial to
quantify how much these quantities deviate from their expectations. We now prove
two concentration inequalities: a bounded-difference bound (McDiarmid) for \(T_k\),
and a Hoeffding bound for \(\bar u^{(k)}\).

\begin{proposition}[Concentration of \(T_k\) and \(\bar u^{(k)}\)]
\label{prop:concentration}
Let \(D_k := \sup_{z,z'\in \operatorname{supp}(\tilde P_k^Z)}\|z-z'\|_1 <\infty\) be the \(\ell_1\)-diameter
of the embedding support. Define \(T_k=\mathcal W_1(\tilde P_k^Z,\hat{\tilde P}_{k,n}^Z)\),
and \(\bar u^{(k)} = \frac{1}{n}\sum_{i=1}^n u_i^{(k)}\) with
\(u_i^{(k)}=\|y_i-F_k^*(z_{k,i})\|_1\). Assume \(u_i^{(k)}\in[0,2]\) for all \(i\)
(normalization as in the main text). Then for any \(\delta\in(0,1)\),
with probability at least \(1-\tfrac{\delta}{2(L+1)}\),
\begin{align}
T_k &\le \mathbb E[T_k] + D_k\sqrt{\tfrac{1}{2n}\log\tfrac{2(L+1)}{\delta}}, \label{eq:Tk_conc}\\
\bar u^{(k)} &\le \mathbb E[u^{(k)}] + \sqrt{\tfrac{2}{n}\log\tfrac{2(L+1)}{\delta}}. \label{eq:uk_conc}
\end{align}
\end{proposition}

\begin{proof}

\textbf{Step 1: Bounded-difference inequality for \(T_k\).}
We use the Kantorovich-Rubinstein dual representation of \(W_1\):
\[
\mathcal W_1(\mu,\nu) = \sup_{\substack{f:\mathbb R^d\to\mathbb R\\ \mathrm{Lip}(f)\le 1}}
\Big\{\int f\,d\mu - \int f\,d\nu\Big\},
\]
with Lipschitz constant measured in the \(\ell_1\)-norm.  
Let the empirical measure be
\(\hat{\tilde P}_{k,n}^Z = \tfrac{1}{n}\sum_{i=1}^n\delta_{z_{k,i}}\).
Consider two samples \(S=(z_{k,1},\dots,z_{k,n})\) and \(S^{(j)}\) that differ
only in the \(j\)-th element. Denote
\(T_k(S)=\mathcal W_1(\tilde P_k^Z,\hat{\tilde P}_{k,n}^Z(S))\).  
Then
\[
|T_k(S)-T_k(S^{(j)})|
\le \tfrac{1}{n}\|z_{k,j}-z_{k,j}'\|_1 \le \tfrac{D_k}{n}.
\]
Thus \(T_k\) satisfies a bounded-difference property with sensitivity \(D_k/n\).
Applying McDiarmid's inequality gives, for any \(t>0\),
\[
\mathbb P\big(T_k - \mathbb E[T_k] \ge t\big)
\le \exp\!\Big(-\tfrac{2n t^2}{D_k^2}\Big).
\]
Choosing \(t = D_k\sqrt{\tfrac{1}{2n}\log\tfrac{2(L+1)}{\delta}}\) yields
\eqref{eq:Tk_conc}.

\medskip
\textbf{Step 2: Hoeffding bound for \(\bar u^{(k)}\).}
Each \(u_i^{(k)}\in[0,M]\) by Assumption \ref{assump:bayes_bounded}. By Hoeffding's inequality, for any \(t>0\),
\[
\mathbb P\big(\bar u^{(k)} - \mathbb E[\bar u^{(k)}] \ge t\big)
\le \exp\!\Big(-\tfrac{n t^2}{2}\Big).
\]
Choosing \(t = \sqrt{\tfrac{2}{n}\log\tfrac{2(L+1)}{\delta}}\) yields
\eqref{eq:uk_conc}.

\medskip
This completes the proof.
\end{proof}

\paragraph{Discussion.}
This result ensures that both the statistical fluctuation of the embedding
distribution (through \(T_k\)) and the empirical noise magnitude (through
\(\bar u^{(k)}\)) remain close to their expectations with high probability.
These concentration bounds are the key probabilistic ingredients needed to
convert the deterministic decomposition of the generalization gap into a
high-probability generalization bound.

\subsubsection{Combined deterministic and high-probability bound}

\paragraph{Motivation.}
We now combine the pieces developed above.  
Recall that the observed generalization gap
\[
R_0^{\mathrm{obs}} - \hat R_{0,n}^{\mathrm{obs}}
\]
was decomposed into three terms (Proposition~\ref{prop:decomp_app}).  
We provided deterministic bounds for each term (Lemmas~\ref{lem:gap_A}--\ref{lem:gap_C}), 
and then concentration inequalities for the random quantities \(T_k\) and \(\bar u^{(k)}\) (Proposition~\ref{prop:concentration}).  
Here we integrate these ingredients into a single high-probability generalization bound.

\begin{proposition}[High-probability control of the generalization gap]
\label{prop:combined_bound}
Assume Assumptions~\ref{assump:measurable}--\ref{assump:bayes_bounded}.  
Suppose that for each layer \(k\) there exist constants \(C_k>0\), arbitrarily small $\epsilon>0$ and \(d_k>0\) such that
\(\mathbb E[T_k]\le C_k n^{-1/(d_k+\epsilon)}\) for all sufficiently large \(n\).  
Fix confidence \(\delta\in(0,1)\). Then with probability at least \(1-\delta\), 
simultaneously for all layers \(k=0,\dots,L\) and any fixed predictors \(F\in\mathcal F\),
\begin{align}
R_0^{\mathrm{obs}} - \hat R_{0,n}^{\mathrm{obs}}
&\le L_k(g)\Big(C_k n^{-1/(d_k+\epsilon)} + D_k\sqrt{\tfrac{1}{2n}\log\tfrac{2(L+1)}{\delta}}\Big) \nonumber\\
&\quad + M_{F^*}\Big(2\,\mathbb E\|Y-F_k^*(Z)\|_1 + \sqrt{\tfrac{2}{n}\log\tfrac{2(L+1)}{\delta}}\Big).
\label{eq:combined_highprob_full}
\end{align}
In terms of rates and up to multiplicative constants, the bound can be summarized as
\[
\fbox{$\displaystyle
R_0^{\mathrm{obs}} - \hat R_{0,n}^{\mathrm{obs}}
\;\;\lesssim\;\;
L_k(g)\,n^{-1/(d_k+\epsilon)} \;+\; M_{F^*}\,\mathbb E\|Y-F_k^*(Z)\|_1 
\;+\; \sqrt{\tfrac{\log(L/\delta)}{n}}\,\big(L_k(g)D_k + M_{F^*}\big)
$}
\]
\end{proposition}

\begin{proof}
\textbf{Step 1: Decomposition.}  
By Proposition~\ref{prop:decomp_app},
\[
R_0^{\mathrm{obs}} - \hat R_{0,n}^{\mathrm{obs}}
= (A) + (B) + (C).
\]

\textbf{Step 2: Deterministic bounds.}  
From Lemmas~\ref{lem:gap_A} -- \ref{lem:gap_C},
\[
R_0^{\mathrm{obs}} - \hat R_{0,n}^{\mathrm{obs}}
\le M_{F^*}\,\mathbb E\|Y-F_k^*(Z)\|_1 \;+\; L_k(g)\,T_k \;+\; M_{F^*}\,\bar u^{(k)}.
\]

\textbf{Step 3: Concentration.}  
By Proposition~\ref{prop:concentration}, with probability at least \(1-\delta\),
\[
T_k \le \mathbb E[T_k] + D_k\sqrt{\tfrac{1}{2n}\log\tfrac{2(L+1)}{\delta}}, 
\qquad
\bar u^{(k)} \le \mathbb E[u^{(k)}] + \sqrt{\tfrac{2}{n}\log\tfrac{2(L+1)}{\delta}}.
\]
Since \(\mathbb E[u^{(k)}]=\mathbb E\|Y-F_k^*(Z)\|_1\), substituting yields
\begin{align*}
R_0^{\mathrm{obs}} - \hat R_{0,n}^{\mathrm{obs}}
&\le L_k(g)\Big(\mathbb E[T_k] + D_k\sqrt{\tfrac{1}{2n}\log\tfrac{2(L+1)}{\delta}}\Big) \\
&\quad + M_{F^*}\Big(2\mathbb E\|Y-F_k^*(Z)\|_1 + \sqrt{\tfrac{2}{n}\log\tfrac{2(L+1)}{\delta}}\Big).
\end{align*}
Finally substitute \(\mathbb E[T_k]\le C_k n^{-1/(d_k+\epsilon)}\) to obtain \eqref{eq:combined_highprob_full}.
\end{proof}

\paragraph{Discussion.}
This bound highlights three components:
\begin{itemize}
  \item The \emph{statistical rate} \(L_k(g)\,C_k n^{-1/(d_k+\epsilon)}\) combines embedding geometry (via \(d_k\)) and oracle loss sensitivity (via \(L_k(g)\)).
  \item The \emph{noise/approximation terms} \(M_{F^*}\,\mathbb E\|Y-F_k^*(Z)\|_1\) arise from replacing discrete labels by the Bayes predictor.
  \item The \emph{concentration terms} scale as \(O(\sqrt{\tfrac{\log(L/\delta)}{n}})\), with constants depending on both distributional (\(M_{F^*}\)) and geometric (\(D_k\)) quantities.
\end{itemize}
Together, these yield an explicit and interpretable high-probability upper bound on the observed generalization gap.


\subsection{Recovering Network Effects via Lipschitz Constants}
\label{app:network_effects}

In the previous subsection, the oracle statistical gap (B) was controlled using the Lipschitz constant \(L_k(g)\) of the oracle loss.   
We now expand it to expose how the bound depends both on the network architecture (controllable) and on the data distribution (intrinsic).

\subsubsection{Expansion of \(L_k(g)\)}

From Lemma~\ref{lem:grad_bound_app},
\[
L_k(g) := \sup_{z\in U_k}\|\nabla g_k(z)\|_\infty
\;\le\; L_k(F)\,M_F + L_k(F^*)\,M_{F^*},
\]
where:
\begin{itemize}[leftmargin=*]
    \item \(L_k(F)\) is the Lipschitz constant of the tail sub-network from layer \(k\) to the output;
    \item \(L_k(F^*)\) is the Lipschitz constant of the Bayes predictor at layer \(k\);
    \item \(M_F\), \(M_{F^*}\) are uniform derivative bounds of the loss with respect to its two arguments.
\end{itemize}

\paragraph{Proof sketch.}
By the chain rule,
\[
\nabla g_k(z) 
= \nabla F_k(z)^\top \,\partial_F \ell(F_k(z),F_k^*(z))
+ \nabla F_k^*(z)^\top \,\partial_{F^*}\ell(F_k(z),F_k^*(z)).
\]
Applying the operator norm inequality and the uniform derivative bounds yields the stated inequality.

\subsubsection{Controllable vs.\ intrinsic contributions}

This decomposition separates the two sources of sensitivity:
\begin{itemize}[leftmargin=*]
    \item \textbf{Network-dependent term:}  
    \(L_k(F)\,M_F\), which is determined by the architecture and training of the tail network.  
    It can be reduced by explicit design choices (e.g., normalization layers, spectral norm constraints, Lipschitz regularization).
    
    \item \textbf{Distribution-dependent term:}  
    \(L_k(F^*)\,M_{F^*}\), which reflects the smoothness of the Bayes predictor relative to embeddings.  
    This term is intrinsic to the data distribution and cannot be improved by network design.
\end{itemize}

\subsubsection{Implication for the generalization bound}

Substituting the decomposition of \(L_k(g)\) into Proposition~\ref{prop:combined_bound} gives

\begin{align}
R_0^{\mathrm{obs}} - \hat R_{0,n}^{\mathrm{obs}}
&\le \big(L_k(F)\,M_F + L_k(F^*)\,M_{F^*}\big)\,
\Big(C_k n^{-1/(d_k+\epsilon)} + D_k\sqrt{\frac{1}{2n}\log\frac{2(L+1)}{\delta}}\Big) \nonumber\\
&\quad + \underbrace{\Big[ M_{F^*}\big(2\mathbb E\|Y-F_k^*(Z)\|_1 + \sqrt{\tfrac{2}{n}\log\tfrac{L}{\delta}}\big) \Big]}_{\text{Bayes surrogate terms}}.
\label{eq:combined_highprob_full_split}
\end{align}

Thus the final bound reflects two complementary mechanisms:
\begin{enumerate}[leftmargin=*]
    \item \emph{Embedding geometry:} the intrinsic dimension \(s_k\) governs the statistical rate of Wasserstein convergence;
    \item \emph{Network design:} the Lipschitz constant \(L_k(F)\) controls how embedding perturbations are amplified through the network;
    \item \emph{Bayes surrogate terms:} a residual contribution capturing the discrepancy between discrete labels and their Bayes predictor surrogate, including irreducible randomness.
\end{enumerate}

\section{Details of Experiments}
\label{detail: Experiments}

\subsection{Details of Section \ref{sec:synthetic_wass}}
\label{detail: synthetic_wass}

We conducted an experiment on MNIST to study how the Wasserstein distance between empirical embedding distributions depends on (i) the intrinsic dimension of the embeddings and (ii) the number of samples used to estimate the distributions.

\paragraph{Model and training.}
We trained simple fully connected autoencoders with a symmetric architecture. The encoder flattened each \(28\times28\) image and mapped it to 256 hidden units with ReLU activation, followed by a linear layer to a \(d\)-dimensional bottleneck. The decoder mirrored this with a linear layer back to 256 units, ReLU, and a final linear layer to 784 units. Training used mean squared error loss, the Adam optimizer with learning rate \(10^{-3}\), batch size 128, and 30 epochs. Global randomness was controlled by setting a fixed seed for both PyTorch and NumPy.

\paragraph{Data and embeddings.}
All data were drawn from the MNIST dataset. For the analysis of intrinsic dimension, we trained autoencoders with bottleneck sizes \(d\in\{16,32,64,128,256,512\}\). For the analysis of sample size, we trained a single autoencoder with bottleneck dimension 64 and repeatedly drew two independent subsets of size \(n\in\{100,200,\dots,1000\}\) to evaluate how the Wasserstein distance scales with \(n\). In all cases, embeddings from the training split were used as the empirical distribution, and embeddings from the test split were used as the population distribution.

\paragraph{Intrinsic dimension estimation.}
We estimated the intrinsic dimension using the maximum likelihood estimator of Levina and Bickel \citep{levina2004maximum}, implemented in \texttt{skdim}. 

\paragraph{Wasserstein distance.}
We measured discrepancies between embedding sets using an entropically regularized optimal transport cost (Sinkhorn distance). Uniform weights were assigned to all points, the ground cost was the Euclidean distance, and the regularization parameter was \(\varepsilon=10^{-2}\). Iterations terminated either after 200 steps or once the update magnitude fell below \(10^{-6}\). The resulting cost was computed as the expectation of the ground cost under the transport plan.

\subsection{Details of Section \ref{sec:cross_arch}}
\label{detail: cross_arch}

We conducted experiments on CIFAR-10 and CIFAR-100 to analyze how the final-layer embeddings relate to class-wise generalization gaps modified across ResNet architectures.

\paragraph{Model and training.} 
We considered five ResNet architectures: ResNet-18, 34, 50, 101, and 152. Each model was initialized with ImageNet-pretrained weights from \texttt{torchvision.models} and evaluated on CIFAR datasets. The architecture of these ResNet models was modified by adjusting the final linear output layer. Specifically, the output of the model's convolutional layers is initially projected to a 128-dimensional space via a linear layer. This is then followed by a Sigmoid activation function, and finally, another projection layer yields the ultimate output. These nets are finetuning on Cifar-10 and Cifar-100 used the Adam optimizer with weight decay $0.001$, base learning rate $10^{-4}$, and a cosine annealing schedule over 50 epochs. Batch size was 256, with random horizontal flip for augmentation. Multi-GPU training was enabled via \texttt{accelerate}. Models were saved after training and evaluated on the full test set.

\paragraph{Embedding extraction.} 
For each trained model, we extracted the \emph{last layer embeddings} for all samples in both validation and test splits. Embeddings were stored separately for each class to allow class-wise analysis. For the CIFAR-10 dataset, each class of embeddings in both the validation and test sets comprises 500 samples. In the case of the CIFAR-100 dataset, each class of embeddings in both the validation and test sets consists of 100 samples.

\paragraph{Intrinsic dimension estimation.} 
We estimated the intrinsic dimension of these embeddings using the maximum likelihood estimator of Levina and Bickel \citep{levina2004maximum}, as implemented in \texttt{skdim}. Estimates were computed independently for each class and averaged across samples, yielding 10 estimates per model on CIFAR-10 and 100 per model on CIFAR-100.

\paragraph{Wasserstein distance.} 
For each class, we computed the Wasserstein distance between validation and test embeddings. This used entropic-regularized optimal transport (Sinkhorn distance) with Euclidean ground cost, uniform weights, and regularization parameter $\epsilon=10^{-2}$. These distances quantify how far apart the validation and test embedding distributions are.

\paragraph{Generalization gap.} 
For each class and model, validation and test losses were recorded to compute the class-wise generalization gap.

\subsection{Details of Section \ref{sec:causal_width}}
\label{detail: causal_width}

We designed a experiment on MNIST to analyze how the width of a hidden layer influences intrinsic dimension of intermediate embeddings, Lipschitz properties of the network and generalization performance. The experiment uses a six-layer multilayer perceptron (MLP) with configurable hidden-layer widths and records both statistical and geometric properties of representations throughout training.

\paragraph{Model and training.}
The model is a fully connected network with architecture
\[
784 \to h_1 \to h_2 \to h_3 \to h_4 \to h_5 \to 10,
\]
where each hidden layer is followed by a ReLU activation. The default hidden width is 100 units for all layers. To study the effect of representation bottlenecks, we varied the width of the third hidden layer (\(h_3\)) over the list \(\{100,90,80,70,60,50,40,30,20,10\}\), while keeping all other layers fixed at 100. Training was performed with cross-entropy loss, the Adam optimizer (learning rate \(10^{-3}\), weight decay 0), batch size 128, and for 10 epochs. We used both validation and test splits of MNIST, with additional evaluation on a fixed random subset of 2048 validation samples. All randomness was controlled by fixed seeds and deterministic settings in PyTorch to ensure reproducibility.

\paragraph{Activation collection and intrinsic dimension.}
To measure representation complexity across layers we registered forward hooks after each ReLU activation. During evaluation, the hooks collected activations for all inputs in the 2048-sample subset. For each layer's activation matrix \(X\), we applied the maximum likelihood estimator (as implemented in \texttt{skdim}). 

\paragraph{Lipschitz estimation.}
To characterize stability of the mapping from each hidden layer to the output, we computed the product of spectral norms of all subsequent linear layers. For a given suffix starting at layer \(i\), the Lipschitz constant was approximated by
\[
L_{i \to \text{end}} \;=\; \prod_{j=i+1}^{L} \sigma_{\max}(W_j),
\]
where \(W_j\) denotes the weight matrix of linear layer \(j\) and \(\sigma_{\max}\) is its top singular value. Singular values were computed using \texttt{torch.linalg.svdvals} in double precision. These suffix-wise Lipschitz estimates were recorded at initialization and after each epoch.

\section{Dimensionality Estimation and Hyperparameter Analysis}
\label{supple: algorithm and hyperparameter}

In this appendix, we investigate the effects of hyperparameter choices and the specific algorithm used on the estimation of embedding dimensionality. All experiments are conducted using subsets of the CIFAR datasets: 500 samples per class for CIFAR-10 and 100 samples per class for CIFAR-100.

\subsection{Hyperparameter Analysis}


We first examine how the choice of the hyperparameter $K$ affects dimensionality estimates. Here, $K$ corresponds to the number of nearest neighbors used in the estimation procedure: larger $K$ values capture dimensionality over a broader range of the data, whereas smaller $K$ values reflect more local structure.

\begin{figure}[htbp] 
  \centering
  \includegraphics[width=0.99\linewidth]{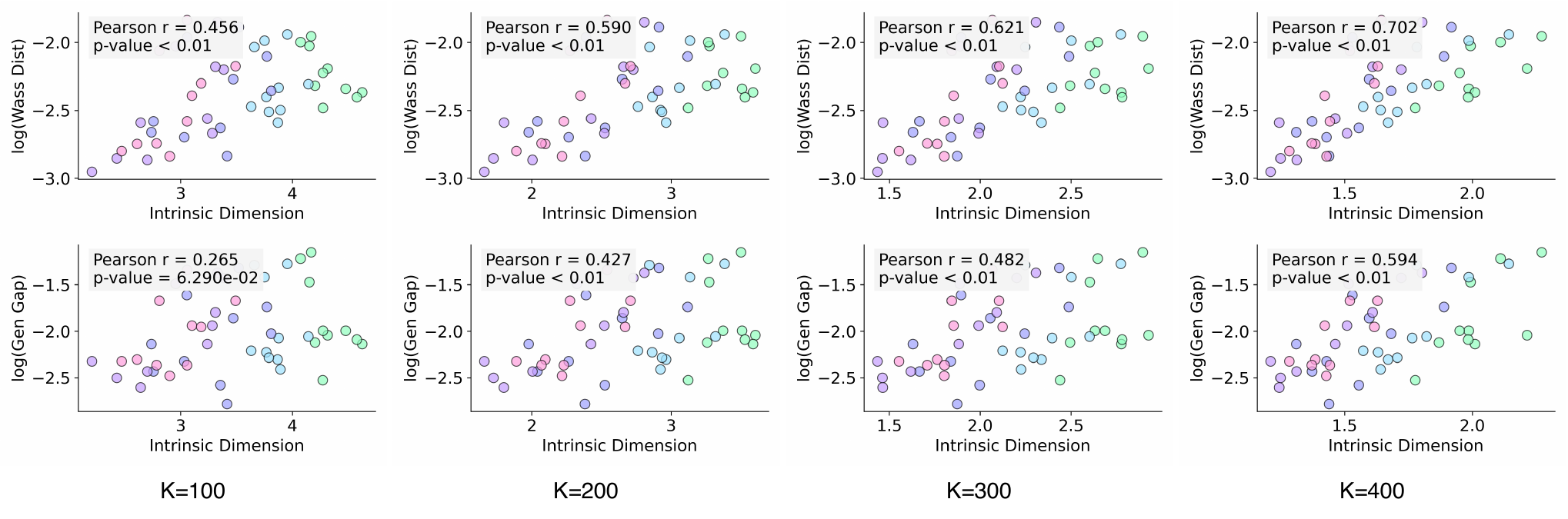}
  \caption{\textbf{Effect of hyperparameter $K$ on dimensionality estimation for CIFAR-10 embeddings.} Larger $K$ values capture broader data structure and lead to higher correlation with generalization error.}
  \label{fig:supple_cifar10_validK}
\end{figure}

\begin{figure}[htbp] 
  \centering
  \includegraphics[width=0.99\linewidth]{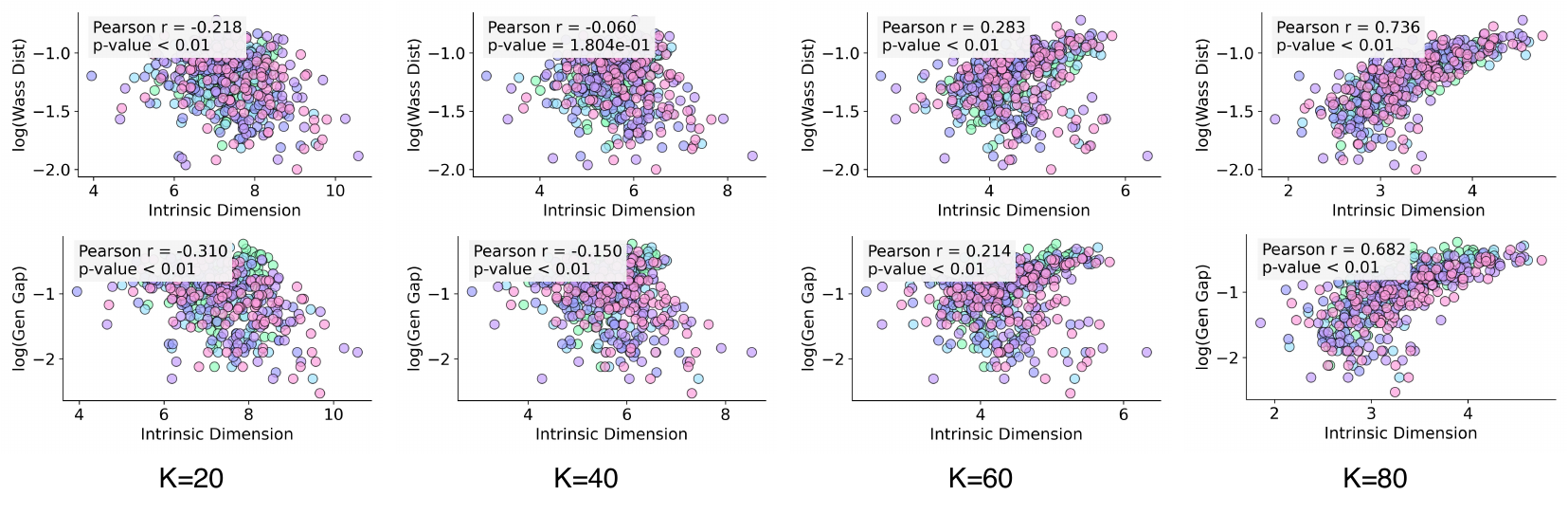}
  \caption{\textbf{Effect of hyperparameter $K$ on dimensionality estimation for CIFAR-100 embeddings.} Increasing $K$ improves the alignment between estimated dimensionality and generalization error, indicating that global structure is more informative.}
  \label{fig:supple_cifar100_validK}
\end{figure}

For CIFAR-10, we test $K = 100, 200, 300, 400, 500$, and for CIFAR-100, we test $K = 20, 40, 60, 80, 100$. Figures~\ref{fig:supple_cifar10_validK} and \ref{fig:supple_cifar100_validK} show the results. We observe that as $K$ increases, the estimated dimensionality better correlates with the generalization error. This indicates that the global dimensionality of a class is more predictive of generalization performance than local dimensionality.

\subsection{Algorithm Comparison}


Next, we compare different dimensionality estimation algorithms (TLE \citep{amsaleg2019intrinsic} and MOM \citep{amsaleg2018extreme}) while keeping the hyperparameter fixed ($K=400$ for CIFAR-10, $K=80$ for CIFAR-100).

\begin{figure}[htbp] 
  \centering
  \includegraphics[width=0.8\linewidth]{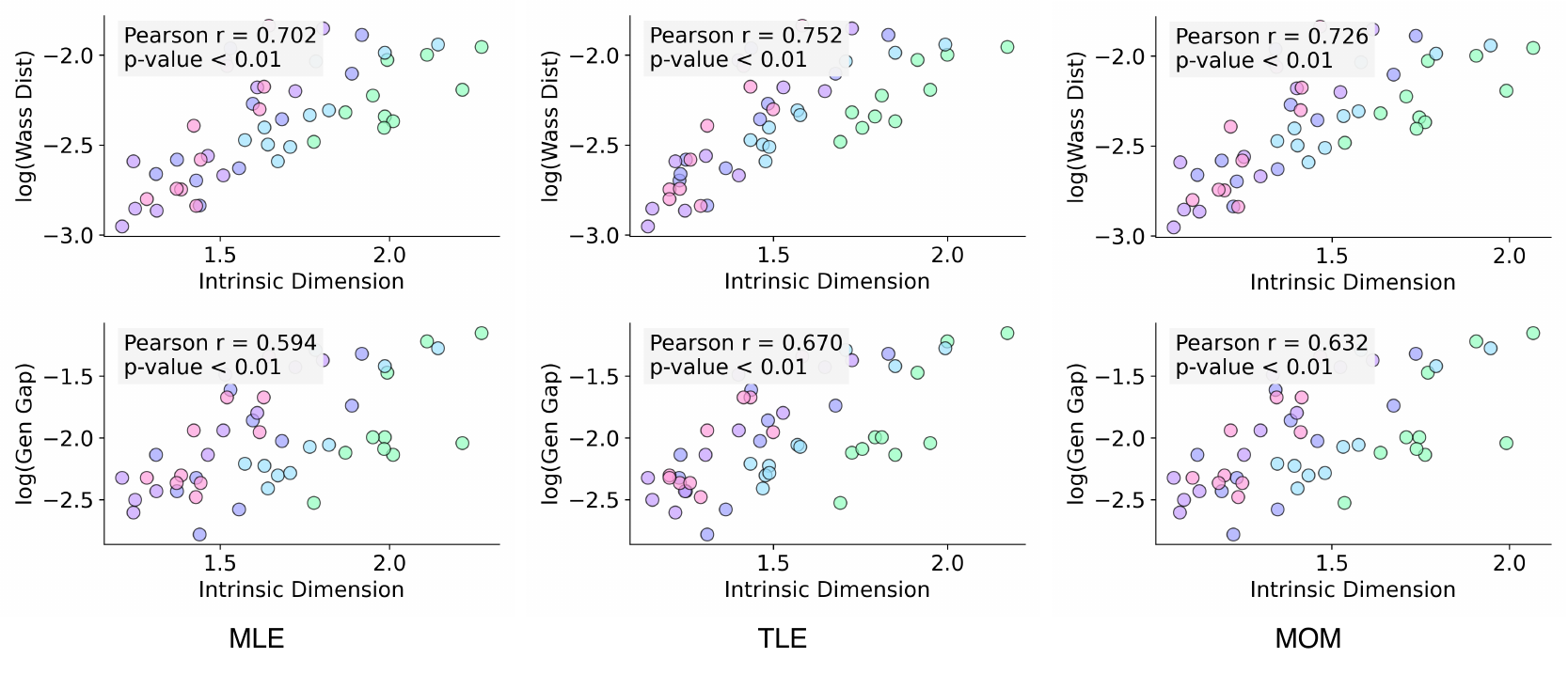}
  \caption{\textbf{Comparison of dimensionality estimation algorithms on CIFAR-10 embeddings.} Despite using different algorithms, estimated dimensionalities consistently correlate with generalization error, demonstrating robustness to method choice.}
  \label{fig:supple_cifar10_validAlgorithm}
\end{figure}

\begin{figure}[htbp] 
  \centering
  \includegraphics[width=0.8\linewidth]{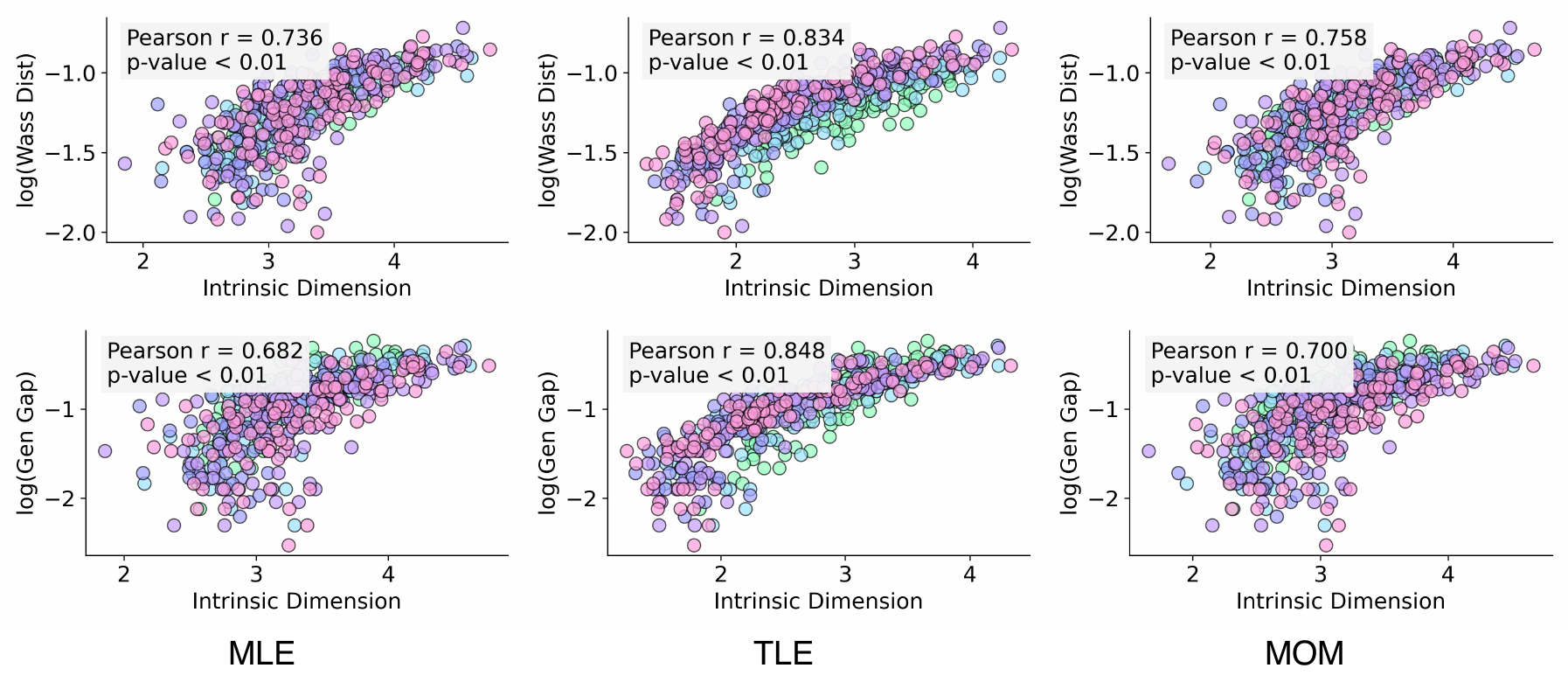}
  \caption{\textbf{Comparison of dimensionality estimation algorithms on CIFAR-100 embeddings.} Dimensionality estimates remain significantly associated with generalization error across different algorithms.}
  \label{fig:supple_cifar100_validAlgorithm}
\end{figure}

Figures~\ref{fig:supple_cifar10_validAlgorithm} and \ref{fig:supple_cifar100_validAlgorithm} present the results. Across both datasets, all algorithms yield estimated dimensionalities that remain significantly correlated with generalization error, suggesting that the observed relationship is robust to the choice of estimation method.

\subsection{Details of Large-Scale Pretrained Models}
\label{appendix: largemodel}

We provide details on the large-scale pretrained models used to evaluate embedding geometry, covering both vision and language modalities.  

\paragraph{Model families and pretraining regimes.}

We consider two modalities:

\begin{table}[htbp]
\centering
\caption{Top-1 accuracy of large-scale pretrained ConvNeXt models on ImageNet-1K.
Models span different architectural scales and pretraining regimes, including
ImageNet-12K, ImageNet-22K, and CLIP-style LAION pretraining.}
\label{tab:imagenet_models}
\begin{tabular}{l c}
\toprule
\textbf{Model} & \textbf{Top-1 Accuracy (\%)} \\
\midrule
ConvNeXt-Nano (IN-12K $\rightarrow$ IN-1K)      & 82.88 \\
ConvNeXt-Tiny (IN-12K $\rightarrow$ IN-1K)      & 84.19 \\
ConvNeXt-Tiny (IN-22K $\rightarrow$ IN-1K)      & 84.10 \\
ConvNeXt-Small (IN-12K $\rightarrow$ IN-1K)     & 85.32 \\
ConvNeXt-Small (IN-22K $\rightarrow$ IN-1K)     & 85.75 \\
ConvNeXt-Base (CLIP LAION $\rightarrow$ IN-1K) & 86.18 \\
ConvNeXt-Large-MLP (CLIP LAION $\rightarrow$ IN-1K) & 87.34 \\
ConvNeXt-Large (IN-22K $\rightarrow$ IN-1K)     & 87.46 \\
ConvNeXt-XLarge (IN-22K $\rightarrow$ IN-1K)    & 87.37 \\
ConvNeXt-XXLarge (CLIP LAION Soup $\rightarrow$ IN-1K) & 88.62 \\
\bottomrule
\end{tabular}
\end{table}

\begin{table}[htbp]
\centering
\caption{Accuracy of pretrained language models fine-tuned on the MNLI benchmark.
Models cover a wide range of architectures and parameter scales, including
BERT, DeBERTa, DistilBERT, DistilBART, and ALBERT variants.}
\label{tab:mnli_models}
\begin{tabular}{l c}
\toprule
\textbf{Model} & \textbf{MNLI Accuracy (\%)} \\
\midrule
BERT-Tiny                                   & 60.00 \\
DistilBERT-Base                             & 82.00 \\
SciBERT                                     & 83.45 \\
BERT-Base                                  & 84.20 \\
ALBERT-Base-v2                              & 85.01 \\
BERT-Large                                 & 86.40 \\
DistilBART (12-1)                           & 87.08 \\
DistilBART (12-3)                           & 88.10 \\
DistilBART (12-6)                           & 89.19 \\
DistilBART (12-9)                           & 89.56 \\
DeBERTa-Large                               & 91.30 \\
DeBERTa-XLarge                              & 91.50 \\
\bottomrule
\end{tabular}
\end{table}

This collection allows us to probe the geometric--generalization relationship
across modalities, architectures, and pretraining scales.

\paragraph{Methodological considerations.}
\begin{itemize}[leftmargin=*]
    \item Embeddings from large pretrained models are not strictly i.i.d.\ due to prior training on massive datasets.  
    \item To approximate the Wasserstein discrepancy between empirical and population embedding distributions, we split the data into two disjoint subsets: one treated as the empirical sample, the other as a proxy for the population distribution.  
    \item Training-set performance is not accessible, so generalization gaps are not directly computed; test accuracy is used as a proxy for downstream performance.  
\end{itemize}

\paragraph{Metrics.}
Wasserstein distances are computed between disjoint subsets of embeddings, separately for vision and language models, to estimate finite-sample approximation error under a fixed pretrained model.

\subsection{Layer-wise Correlations among Dimension, Wasserstein Distance and Generalization Performance}
\label{appendix: layerwise}

We analyzed embeddings from ResNet-152 at layers 4, 18, 30, 43, 55, 67, 79, 91, 103, 115, 127, 139, and 152, and computed the correlation between embedding dimensionality, Wasserstein distances on the validation and test sets, and generalization error.

Correlations are relatively weak in early layers but increase in deeper layers, with a pronounced rise after layer 140. This suggests that deeper embeddings more faithfully capture features relevant to generalization.

\begin{figure}[htbp]
\centering
\includegraphics[width=0.5\linewidth]{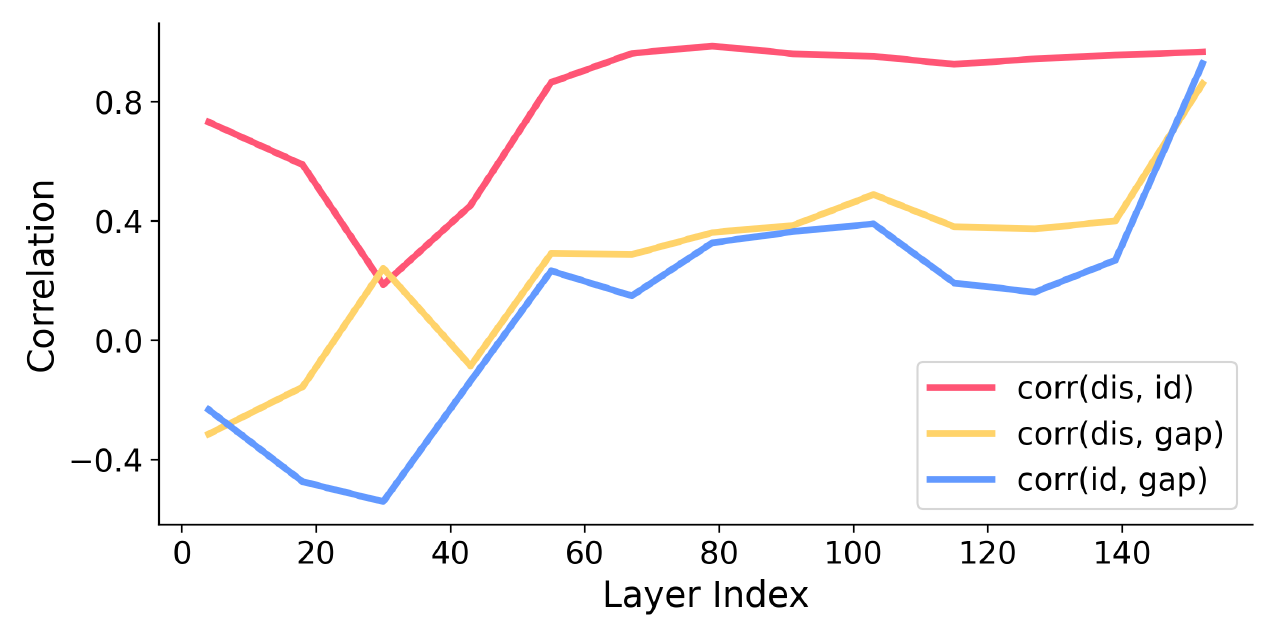}
\caption{\textbf{Layer-wise correlations between embedding dimensionality, Wasserstein distance, and generalization error in ResNet-152.} Deeper layers exhibit stronger correlations, indicating the increasing alignment between representation properties and generalization.}
\label{fig:supple_layer_corr}
\end{figure}







\end{document}